\documentclass[11pt,letterpaper]{article}

\usepackage{silence}
\usepackage[margin=1in]{geometry}
\usepackage[T1]{fontenc}
\usepackage{times}
\usepackage[authoryear,round]{natbib}
\setcitestyle{citesep={;},aysep={,},yysep={;}}

\usepackage{amsmath,amsfonts,bm}

\renewcommand{\eqref}[1]{\textup{(\ref{#1})}}

\usepackage{url}
\usepackage{amssymb}
\usepackage{bbm}
\usepackage{dsfont}
\usepackage{amsthm}
\usepackage{xcolor}
\usepackage{algorithm}
\usepackage{algorithmic}
\usepackage{booktabs}
\usepackage{hyperref}

\theoremstyle{plain}
\newtheorem{theorem}{Theorem}[section]
\newtheorem{lemma}[theorem]{Lemma}
\newtheorem{proposition}[theorem]{Proposition}
\newtheorem{corollary}[theorem]{Corollary}
\theoremstyle{definition}
\newtheorem{definition}[theorem]{Definition}
\newtheorem{assumption}[theorem]{Assumption}
\theoremstyle{remark}
\newtheorem{remark}[theorem]{Remark}

\newcommand{\EE}{\mathbb{E}}
\newcommand{\PP}{\mathbb{P}}
\newcommand{\cW}{\mathcal{W}}
\newcommand{\cF}{\mathcal{F}}
\newcommand{\cL}{\mathcal{L}}
\newcommand{\cZ}{\mathcal{Z}}
\newcommand{\cG}{\mathcal{G}}
\newcommand{\inner}[2]{\langle #1, #2 \rangle}

\newcommand{\ind}[1]{\mathbbm{1}_{#1}}
\newcommand{\mbar}{\overline{m}}
\newcommand{\bF}{b_F}

\title{High-Probability Guarantees for SGD under $\beta$-Heavy-Tailed Gradient Noise}

\author{Qijun Tong\textsuperscript{1} \quad
  Masahiro Ikeda\textsuperscript{2} \quad
  Ryota Kawasumi\textsuperscript{3} \\[0.5em]
  \small\textsuperscript{1}The University of Electro-Communications \\
  \small\textsuperscript{2}The University of Osaka \\
  \small\textsuperscript{3}Gunma University
}
\date{}

\begin{document}

\maketitle

\begin{abstract}
  \noindent Stochastic gradient descent (SGD) is widely used to train machine learning
  models, but subsampling the training data introduces noise into its updates.
  The strength and applicability of high-probability guarantees therefore
  depend critically on how the tails of gradient noise are modeled.
  Reports of heavy-tailed gradient noise in deep learning motivate relaxing
  the bounded-noise and sub-Gaussian assumptions commonly used in
  high-probability analyses of SGD. We use Young
  functions from Orlicz space theory to describe noise tails in a common
  framework. We model SGD gradient noise by adopting a Young function that
  preserves the finiteness of all polynomial moments while allowing tails
  heavier than sub-Weibull, including lognormal distributions. The resulting
  class is called $\beta$-heavy-tailed, with $\beta$ controlling the tail
  heaviness. We establish concentration inequalities for $\beta$-heavy-tailed
  noise and combine them with a uniform bound on the difference
  between empirical and population gradients along the SGD trajectory to obtain
  high-probability bounds on optimization
  and population-risk stationarity for smooth nonconvex losses under
  trajectory assumptions. The bounds
  are not restricted to a particular learning-rate decay rule and make
  explicit the effects of noise tails and learning-rate schedules.
  Under the Polyak--{\L}ojasiewicz condition, we bound the risk at the last
  iterate. We also analyze SGD with gradient clipping under the
  $\beta$-heavy-tailed noise model.
\end{abstract}

\section{Introduction}
\label{sec:intro}

Stochastic gradient descent (SGD) is a principal optimization algorithm used
to train modern machine learning models \citep{bottou2018optimization}, and
characterizing its generalization behavior is a central problem in learning
theory. SGD reduces computational cost by computing a gradient from a subset
of the training data at each step. This introduces a random difference
between that gradient and the full-batch gradient over the entire training
set, referred to as gradient noise. Quantifying how this noise affects
convergence and learning through successive iterations requires control of
its magnitude and the probability of large deviations. Existing analyses
therefore impose assumptions such as bounded variance, with boundedness or
sub-Gaussian tail conditions often used to obtain high-probability
guarantees. However, heavy-tailed gradient noise that is poorly described
by a Gaussian approximation has been reported in deep learning
\citep{simsekli2019tail,zhang2020adaptive,imekli:tel-05434401}. For quantities distinct from
gradient noise itself, near-lognormal distributions have also been reported
for the magnitudes of gradients backpropagated through networks
\citep{chmiel2021neural}, and asymptotic lognormality has been established
for gradient norms in randomly initialized fully connected ReLU networks
\citep{hanin2020products}. These
findings motivate analyzing SGD under weaker tail conditions.

Existing analyses of heavy tails take two approaches. One assumes only a
bounded $p$-th moment for some $p\in(1,2]$, allowing infinite variance when
$p<2$. In this setting, gradient clipping limits the magnitude of a
stochastic gradient when its norm exceeds a threshold, reducing the effect
of rare large gradients \citep{zhang2020adaptive}. High-probability
convergence guarantees have been obtained for methods combining clipping
with momentum and normalization \citep{cutkosky2021high}. Clipping,
however, changes the update rule and generally introduces bias into the
gradient estimate. The other approach relaxes sub-Gaussian assumptions to
sub-Weibull conditions and obtains high-probability guarantees for SGD \citep{li2022high,madden2024high}. Lognormal distributions
still fall outside this setting: their tails decay more slowly than
$\exp(-ct^p)$ for every $p>0$ and $c>0$. Although they satisfy the finite
$p$-th moment assumption, they have all polynomial moments, a structure
that analyses based only on low-order moments do not fully use. We study
high-probability guarantees for SGD in this intermediate
regime, which permits tails heavier than sub-Weibull while retaining the
finiteness of all polynomial moments.

To describe this intermediate regime, we use the class of
$\beta$-heavy-tailed random variables introduced by
\citet{chamakh2021orlicz} in their study of Orlicz norms and concentration
inequalities. This class is defined through the Orlicz norm associated with
the Young function $\Psi_\beta(x)=\exp(\log^\beta(x+1))-1$ for $\beta>1$.
Sub-Gaussian and sub-exponential conditions are Orlicz conditions associated
with $\exp(x^2)-1$ and $\exp(x)-1$, respectively; a more slowly growing
Young function allows heavier tails. A random variable satisfying
$\|X\|_{\Psi_\beta}<\infty$ (Definition~\ref{def:betaheavy}) has a tail
probability bounded by an expression of the form
$\exp(-c(\log t)^\beta)$ (Proposition~\ref{thm:tail}). This bound decays
faster than any polynomial rate $t^{-p}$, ensuring that all polynomial
moments are finite, but more slowly than $\exp(-ct^\epsilon)$ for any
$\epsilon>0$, allowing tails beyond the sub-Weibull scale. Smaller values
of $\beta$ give weaker tail conditions, and larger values give stronger
ones. Our approach is to adopt this probabilistic framework as a model
for SGD gradient noise and develop martingale concentration inequalities
that control its accumulation along the trajectory, yielding a
high-probability analysis of SGD.

\paragraph{Contributions.} Our main contributions are as follows.
\begin{itemize}
  \item We formulate a $\beta$-heavy-tailed gradient noise model
    (Assumption~\ref{asm:gradnoise}) and develop concentration tools for
    its analysis. We show that the associated generating functions are
    equivalent to Young functions even when squaring induces
    nonconvexity, enabling control of sums for
    every $\beta>1$ (Lemma~\ref{lem:convexify}). We also prove martingale
    concentration inequalities, including a Freedman-type bound that
    accommodates noise scales and variance proxies depending on the
    past trajectory (Proposition~\ref{thm:freedman}).
  \item For SGD, we combine empirical optimization and uniform gradient
    generalization bounds to control the weighted average squared
    gradient norm of the population risk with high probability
    (Theorems~\ref{thm:main1}--\ref{thm:main3}). The bounds accommodate
    different step-size schedules; we instantiate them for $1/t$,
    $1/\sqrt t$, and cosine decay.
  \item Under a Polyak--{\L}ojasiewicz (PL) condition on the empirical risk
    and a suitable $1/t$-type schedule, we obtain a high-probability
    $1/T$ rate for the last-iterate optimization error, up to
    logarithmic and tail factors
    (Theorem~\ref{thm:pl}). With a PL condition on the population risk
    as well, $T\asymp n$ iterations yield a $d/n$ excess population
    risk rate up to similar factors (Theorem~\ref{thm:pl-excess}).
  \item Under the same noise model, a noise-dependent clipping threshold
    yields a high-probability guarantee without the gradient cap
    (Assumption~\ref{asm:cap}) or a smoothness-based step-size upper
    bound (Theorem~\ref{thm:clip}). The guarantee is quadratic in the
    gradient norm below the threshold and linear above it, retaining
    the same tail factor as the bound for SGD.
\end{itemize}

\paragraph{Organization.} Section~\ref{sec:related-work} reviews related
work, and Section~\ref{sec:setup} introduces the notation and assumptions.
Section~\ref{sec:main-results} presents the main bounds and instantiates
them for specific step-size schedules. Sections~\ref{sec:pl}
and~\ref{sec:clipping} address the PL condition and gradient clipping.

\subsection{Related Work}
\label{sec:related-work}

\paragraph{High-probability guarantees for SGD.}
Nonconvex SGD has been analyzed in expectation
\citep{ghadimi2013stochastic}, while \citet{madden2024high} give
high-probability convergence guarantees under sub-Weibull noise.
The work closest to ours, \citet{li2022high}, analyzes both optimization
and generalization under the same tail condition. We relax this condition
and express the guarantees through the sum of the step sizes and the sum
of their squares, rather than restricting the analysis to the
$1/\sqrt t$ decay used in that work. Approaches to generalization include
algorithmic stability
\citep{bousquet2002stability,hardt2016train,lei2020fine,lei2021learning}
and uniform convergence of empirical to population gradients
\citep{mei2018landscape,foster2018uniform}. To evaluate stationarity, we
use the latter approach through the uniform gradient bound of
\citet{lei2021learning} (Lemma~\ref{lm:gradgen}).
\citet{liu2024revisiting} remove the compact-domain and bounded-noise
restrictions from last-iterate analyses of convex composite objectives under
a stochastic oracle.

Under finite $p$-th moment assumptions, high-probability convergence rates
for methods using gradient clipping have been refined
\citep{sadiev2023high,nguyen2023improved}, and \citet{liu2025nonconvex}
obtain an optimal rate without clipping by using batched normalized SGD
with momentum.
\citet{koloskova2023revisiting} study the bias introduced by clipping and
its convergence guarantees, while \citet{armacki2025high} consider more
general nonlinear transformations. Our clipping analysis concerns noise
with all polynomial moments finite and does not cover the full setting
in which only low-order moments are assumed finite.

\paragraph{Heavy-tailed gradient noise in neural networks.}
The precise tail law of gradient noise remains under debate. Following
the $\alpha$-stable model of \citet{simsekli2019tail},
\citet{xie2023overlooked} reject power-law tails for minibatch gradient
noise, while \citet{battash2024revisiting} revisit the question in support
of an $\alpha$-stable model. \citet{gurbuzbalaban2021heavy} also show
that heavy-tailed stationary distributions of SGD iterates can arise from
the dynamics themselves. Studies of
low-precision training report gradient magnitudes concentrated near zero
and spanning a wide dynamic range, without fitting a specific distribution
family \citep{micikevicius2018mixed,zhu2020towards,sun2020ultra}.
Motivated by these observations, we control noise tails through an Orlicz
condition without assuming a specific distribution family.

\paragraph{Orlicz norms and martingale concentration.}
Deviation bounds based on Orlicz norms have been used to refine
Bernstein-type inequalities \citep{van2011bernstein} and to study
concentration for sub-Weibull variables
\citep{vladimirova2020subweibull,kuchibhotla2022moving}.
\citet{li2024concentration} give a general framework for broader families
of heavy-tailed variables. In the $\beta$-heavy-tailed framework we use,
the concentration inequalities of \citet{chamakh2021orlicz} concern sums
of independent variables. SGD noise $\epsilon_t$, however, depends on the
past sampling history through $\bm{w}_t$, so these results do not apply
directly. We extend the analysis to the martingale setting, following
\citet{freedman1975tail} (Proposition~\ref{thm:freedman}).

\section{Problem Setup}
\label{sec:setup}

\subsection{Notation}
\label{sec:notation}

For $n \in \mathbb{N}$ we write $[n] := \{1,\dots,n\}$, and $\|\cdot\|$ for the
Euclidean norm on $\mathbb{R}^d$. For $x \in \mathbb{R}$ we write
$(x)_+ := \max\{x,0\}$, and $\ind{E}$ for the indicator of an event $E$. For
positive quantities $a$ and $b$ we write $a \asymp b$ if $c\,b \le a \le C\,b$
for constants $0 < c \le C$ that do not depend on $n$, $d$, $T$ or $\delta$. We
abbreviate $\log^\beta(x) := (\log x)^\beta$,
and for $\beta>1$ define
\[
  \Psi_\beta(x) := \exp\!\left(\log^\beta(x+1)\right) - 1, \qquad
  \overline{\Psi}_\beta(x) := \exp\!\left(-\log^\beta(x+1)\right)
  = \frac{1}{\Psi_\beta(x)+1}.
\]
Here $\Psi_\beta$ is the Young function whose Orlicz norm
$\|X\|_{\Psi_\beta}$ measures the tail of a random variable $X$ throughout the
paper. Appendix~\ref{sec:orlicz} recalls the definition of a Young function and
of its Orlicz norm, together with the basic properties we use.

Let $\cZ$ be the sample space, a Borel subset of a Euclidean space, $P$ a
probability distribution on $\cZ$, and $\cW := \mathbb{R}^d$ the parameter
space. The loss function $\ell : \cW \times \cZ \to
\mathbb{R}$ is differentiable in its first argument, and the learning problem is
the minimization over $\cW$ of the population risk $\cL(\bm{w}) := \EE_{z \sim
P}[\ell(\bm{w}; z)]$. We denote its infimum by
$\cL^\star := \inf_{\bm{w} \in \cW}\cL(\bm{w})$.
Since $P$ is accessible only through data, one minimizes
instead the empirical risk $\cL_S(\bm{w}) := \frac{1}{n}\sum_{i\in[n]}
\ell(\bm{w}; z_i)$ over a training set $S = \{z_1,\dots,z_n\}$ of i.i.d.\
samples from $P$. We write $\bm{w}(S)$ for a minimizer of $\cL_S$ over $\cW$,
which we assume to exist. We study the standard stochastic gradient descent iteration of
Algorithm~\ref{alg:sgd} applied to $\cL_S$, where at each step an index $j_t$ is
drawn uniformly from $[n]$ and the iterate is updated along the corresponding
per-sample gradient. Let $T$ denote the total number of SGD updates (the time
horizon). We fix $\bm{w}_1 = \bm{0}$, a normalization that
places the origin at the initialization. For a general $\bm{w}_1$ independent
of $S$, the results below hold with $\|\bm{w}_t\|$ replaced by
$\|\bm{w}_t-\bm{w}_1\|$. The choice of the initialization is not studied here.

\begin{algorithm}[h]
  \caption{Stochastic Gradient Descent}
  \label{alg:sgd}
  \begin{algorithmic}[1]
    \STATE {\bfseries Input:} initial point $\bm{w}_1 = \bm{0}$, step-size sequence
    $\{\eta_t\}_t$, dataset $S = \{z_1,\dots,z_n\}$.
    \FOR{$t = 1,\dots,T$}
    \STATE Draw $j_t$ uniformly from $[n]$.
    \STATE $\bm{w}_{t+1} \leftarrow \bm{w}_t - \eta_t \nabla \ell(\bm{w}_t; z_{j_t})$.
    \ENDFOR
    \STATE {\bfseries Output:} iterates $\bm{w}_1,\dots,\bm{w}_{T+1}$.
  \end{algorithmic}
\end{algorithm}

We work on the joint probability space of the training sample $S$ and the
sampled indices. Define
\[
  \cF_0 := \sigma(S), \qquad
  \cF_t := \sigma(S,j_1,\dots,j_t) \quad (t\ge1).
\]
At each step, $j_t$ is drawn uniformly from $[n]$, independently of
$\cF_{t-1}$. Then $\bm{w}_t$ is $\cF_{t-1}$-measurable, and
$\EE[\nabla \ell(\bm{w}_t; z_{j_t}) \mid \cF_{t-1}] = \nabla \cL_S(\bm{w}_t)$. In
this paper, we focus on the per-step \emph{gradient noise} defined by
\[
  \epsilon_t := \nabla \ell(\bm{w}_t; z_{j_t}) - \nabla \cL_S(\bm{w}_t).
\]
The sequence $(\epsilon_t)_t$ is a martingale difference sequence (MDS) with
respect to $(\cF_t)_t$. The next subsection models the tail of this sequence.

\subsection{$\beta$-Heavy-Tailed Noise Model}
\label{sec:beta-heavy}

We model the tail of a random variable using the
$\Psi_\beta$ ($\beta > 1$) Orlicz norm, following the general framework of
Orlicz-norm-based tail bounds (cf.\ \citet{chamakh2021orlicz}).

\begin{definition}[$\beta$-heavy tail property of a random variable]
  \label{def:betaheavy}
  Let $\beta > 1$. A random variable $X$ is
  \emph{$\beta$-heavy tailed} with $K>0$ if
  \[
    \EE\left[\Psi_\beta\!\left(\frac{|X|}{K}\right)\right] \le 1
    \iff \EE\left[\exp\left(\log^\beta\left(\frac{|X|}{K}+1\right)\right)\right] \le 2.
  \]
  Equivalently, the Orlicz norm is at most $K$, i.e.\ $\|X\|_{\Psi_\beta} \le K$
  (Lemma~\ref{lem:orlicz-props}(iii) in Appendix~\ref{sec:orlicz}).
\end{definition}

The prototypical example is the lognormal distribution. Every lognormal
variable satisfies Definition~\ref{def:betaheavy} for some $K$ whenever
$1<\beta<2$. If $\log X$ is Gaussian with variance $\sigma^2$, the tail of
$X$ decays like
$\exp(-\log^2 t/(2\sigma^2))$, matching the $\beta=2$ tail rate of
Proposition~\ref{thm:tail} up to the constant in the exponent. At $\beta=2$,
Definition~\ref{def:betaheavy} holds for some $K$ if and only if
$\sigma^2\le1/2$, since the scale $K$ does not change the leading
coefficient of $\log^2 t$. Lognormal-type
tails are the motivating example for this Orlicz condition in
\citet{chamakh2021orlicz}. Related empirical motivation
comes from the approximately lognormal magnitudes of gradients
with respect to intermediate layer outputs reported by
\citet{chmiel2021neural}.

Our noise assumption allows the conditional law of $\epsilon_t$ given
$\cF_{t-1}$ to depend on the past trajectory, while a fixed constant $K$
uniformly bounds the conditional $\Psi_\beta$-Orlicz norms of $\|\epsilon_t\|$.

\begin{assumption}[$\beta$-heavy tail assumption on the gradient noise]
  \label{asm:gradnoise}
  There exist deterministic constants $\beta>1$ and $K>0$ such that, for every $t$, the gradient noise
  $\epsilon_t$ satisfies
  \begin{equation*}
    \EE[\epsilon_t \mid \cF_{t-1}] = 0,
    \qquad
    \EE\left[\Psi_\beta\!\left(\frac{\|\epsilon_t\|}{K}\right)
    \;\middle|\; \cF_{t-1}\right] \le 1.
  \end{equation*}
  The Freedman-type inequality of Proposition~\ref{thm:freedman} requires
  the conditional form of the Orlicz bound above.
\end{assumption}

\begin{remark}[Boundedness and noise scales]
  Some readers may wonder what heavy tails mean when gradients sampled
  from a finite dataset are necessarily bounded at a fixed iterate.
  Conditional on the dataset and past iterates, this noise indeed has finite
  support. Assumption~\ref{asm:gradnoise} instead specifies a uniform Orlicz
  scale. Boundedness alone does not compare tail-based guarantees, whose
  scales and constants also matter; the paragraph immediately following
  Remark~7 in \citet{madden2024high} makes this point for bounded
  sub-Weibull noise. Appendix~\ref{sec:bounded-noise-scales} gives the
  corresponding calculation for our $\beta$-heavy-tailed model.
  In contrast to finite-data sampling, a population stochastic oracle can
  produce noise with an unbounded conditional law. In this setting, tail
  assumptions control the probability of large noise realizations without
  requiring an almost-sure bound; \citet[Assumption 1.2 and Theorem 4.4]{zhu2026stochastic}
  analyze such sub-Weibull noise. That formulation directly optimizes the population risk,
  whereas our finite-data formulation retains the empirical-risk and
  generalization questions.
\end{remark}

We also impose two assumptions widely used in the
optimization and generalization analysis of nonconvex learning problems.

\begin{assumption}[Smoothness of the loss function]
  \label{asm:smooth}
  There exists $b>0$ such that
  $\|\nabla \ell(\bm{w};z)-\nabla \ell(\bm{w}';z)\| \le b \|\bm{w}-\bm{w}'\|$ for any
  sample $z \in \cZ$ and parameters $\bm{w}, \bm{w}' \in \cW$.
\end{assumption}

Assumption~\ref{asm:smooth} is the standard regularity condition in nonconvex
stochastic optimization \citep{ghadimi2013stochastic,hardt2016train}. We impose
it on the per-sample loss uniformly in $z$, so that $\cL_S$ and $\cL$ are
$b$-smooth as well, and it yields the descent lemma (Lemma~\ref{thm:smooth}),
the one-step inequality used to derive the trajectory bound of
Theorem~\ref{thm:main1}.
\begin{assumption}
  \label{asm:cap}
  There exists $G>0$ such that $\eta_t \|\nabla \cL_S(\bm{w}_t)\| \le G$ for any
  dataset $S \in \cZ^n$ and $t \in \mathbb{N}$.
\end{assumption}
Assumption~\ref{asm:cap} bounds the product of the step size and the empirical
gradient norm. This is an additional trajectory condition; it does not follow
from Assumption~\ref{asm:gradnoise}, which controls deviations from the
empirical gradient but does not bound the empirical gradient itself.

\section{Main Results}
\label{sec:main-results}

Assumptions~\ref{asm:gradnoise} and \ref{asm:smooth} hold
throughout the rest of the paper, with tail index $\beta>1$, smoothness
constant $b>0$ and $b' := \sup_{z\in\cZ}\|\nabla \ell(\bm{0};z)\| < \infty$,
and the step sizes are positive and nonincreasing.
For SGD without clipping, we additionally assume
Assumption~\ref{asm:cap} and $\eta_t \le 1/(2b)$.
We also assume throughout that there exists a deterministic constant
$\Delta_0<\infty$, independent of $n$ and $d$, such that, for every sample
size $n$, almost surely over $S\sim P^n$,
$\cL_S(\bm{0})-\cL_S(\bm{w}(S))\le\Delta_0$.

The implicit constants in $\mathcal{O}(\cdot)$ are deterministic and may
depend on $\beta$ and fixed problem parameters such as $b,b',K,G,\Delta_0$
and, in the PL results, $\mu_0,\mu$, but not on $S,n,d,T,\delta$.
Proofs of all theorems stated in this section are deferred to the appendices.
\subsection{High-Probability Bound on the SGD Trajectory}
\label{sec:main-trajectory}

Our first result bounds the sum of squared empirical gradient norms along
the SGD trajectory.

\begin{theorem}
  \label{thm:main1}
  For any $\delta\in(0,1)$, with probability $1-\delta$,
  \[
    \sum_{t\in[T]} \eta_t \|\nabla \cL_S(\bm{w}_t)\|^2
    = \mathcal{O}\!\left(\exp\!\left(2\log^{1/\beta}(T/\delta)\right)
    \left(\log(e/\delta)+\sum_{t\in[T]}\eta_t^2\right)\right).
  \]
\end{theorem}
The factor $\exp(2\log^{1/\beta}(\cdot))$ arises from the $\beta$-heavy-tailed noise. The bound degrades as $\beta$ decreases, that is, as the noise model admits heavier tails.

\begin{remark}
  A threshold at $\beta\approx1.519$ affects the constant factors in this
  analysis. Below this threshold, the generating function associated with the
  square of a $\beta$-heavy-tailed random variable loses convexity, requiring an
  additional constant factor when bounding sums. This parallels the dependence
  of sub-Weibull sum bounds on the tail parameter $\theta$, where the generating
  function also loses convexity as the tail condition weakens. The rate stated
  above applies on both sides of the threshold; see
  Remark~\ref{rem:squared-convexity} in Appendix~\ref{sec:orlicz} for details.
\end{remark}

Next, we establish a high-probability bound, uniform over all SGD iterates, on
the squared discrepancy between the population and empirical risk gradients.

\begin{theorem}
  \label{thm:main2}
  For any $\delta\in(0,1)$, with probability $1-\delta$ the
  generalization error of the gradient is bounded simultaneously for all
  iterates as
  \begin{multline}
    \max_{t\in[T+1]}\|\nabla \cL(\bm{w}_t) - \nabla \cL_S(\bm{w}_t)\|^2 \\
    = \mathcal{O}\!\left(\frac{\left(1+\sum_{t\in[T]} \eta_t\right)
      \left(d+\log(1/\delta)\right)}{n}
      \exp\!\left(2\log^{1/\beta}(T/\delta)\right)
    \left(\log(e/\delta)+\sum_{t\in[T]}\eta_t^2\right)\right).
    \label{eq:gen-unif}
  \end{multline}
\end{theorem}

Finally, combining Theorems~\ref{thm:main1} and \ref{thm:main2} bounds the
step-size-weighted average squared norm of the population risk gradient along
the SGD trajectory.

\begin{theorem}
  \label{thm:main3}
  For any $\delta\in(0,1)$, with probability $1-\delta$ the
  step-size-weighted average squared gradient norm of the population risk is
  bounded as
  \begin{multline}
    \frac{\sum_{t\in[T]} \eta_t\|\nabla \cL(\bm{w}_t)\|^2}{\sum_{t\in[T]} \eta_t} \\
    = \mathcal{O}\!\left(\exp\!\left(2\log^{1/\beta}(T/\delta)\right)\!
      \left(\log(e/\delta)+\sum_{t\in[T]}\eta_t^2\right)\!
      \left(\frac{1}{\sum_{t\in[T]}\eta_t}
    + \frac{1+\sum_{t\in[T]}\eta_t}{n}\!\left(d+\log(1/\delta)\right)\right)\right).
    \label{eq:main3}
  \end{multline}
\end{theorem}

The two terms in the last factor of \eqref{eq:main3} correspond to the two sources
of error. The first, $1/\sum_{t\in[T]} \eta_t$, is an optimization term that arises
because SGD is run for only a finite number of steps, and it decreases as the
cumulative step size $\sum_{t\in[T]}\eta_t$ grows. The second,
$(1+\sum_{t\in[T]}\eta_t)(d+\log(1/\delta))/n$, is a generalization term inherited
from Theorem~\ref{thm:main2}. A longer trajectory travels farther from the
initialization, enlarging the radius of the ball over which the uniform
bound of Lemma~\ref{lm:gradgen} must hold. This term therefore grows with
$\sum_{t\in[T]}\eta_t$. The two terms balance when the cumulative step size satisfies
$\sum_{t\in[T]} \eta_t \asymp \sqrt{n/(d+\log(1/\delta))}$, and this choice can be
made without fixing a particular schedule. The following corollary bounds the
step-size-weighted average squared norm of the population risk gradient under
this choice of cumulative step size.

\begin{corollary}
  \label{cor:horizon}
  Under the hypotheses of Theorem~\ref{thm:main3}, assume in addition that
  $n \ge d+\log(1/\delta)$ and that the horizon $T$ is chosen so that
  \begin{equation}
    \sum_{t\in[T]}\eta_t \asymp \sqrt{\frac{n}{d+\log(1/\delta)}}.
    \label{eq:budget}
  \end{equation}
  Then, with probability $1-\delta$,
  \[
    \frac{\sum_{t\in[T]} \eta_t\|\nabla \cL(\bm{w}_t)\|^2}{\sum_{t\in[T]} \eta_t}
    = \mathcal{O}\!\left(\exp\!\left(2\log^{1/\beta}(T/\delta)\right)
      \left(\log(e/\delta)+\sum_{t\in[T]}\eta_t^2\right)
    \sqrt{\frac{d+\log(1/\delta)}{n}}\right).
  \]
\end{corollary}

Condition \eqref{eq:budget} constrains the horizon only through the cumulative
step size, so solving it for $T$ once a schedule is fixed yields a deterministic
iteration number as a function of $n$, $d$, and $\delta$. The balancing cumulative step size is finite
for a finite sample, so the corollary motivates early stopping within this bound.
After imposing \eqref{eq:budget}, the remaining dependence on $T$ is through
$\sum_{t\in[T]}\eta_t^2$ and the tail factor
$\exp(2\log^{1/\beta}(T/\delta))$. For fixed $\beta>1$, this factor grows
more slowly than any positive power of $T/\delta$, yet faster than any
power of $\log(T/\delta)$.

\subsection{Learning-Rate Decay}
\label{sec:instantiation}

Our main theorems depend on the step sizes only through $\sum_{t\in[T]} \eta_t$
and $\sum_{t\in[T]} \eta_t^2$, leaving the decay schedule unspecified. We
instantiate them with three representative schedules, namely the classical $1/t$
decay of stochastic approximation \citep{robbins1951stochastic}, the
$1/\sqrt t$ decay assumed by \citet{li2022high}, and cosine decay, the
single-cycle case of the cosine annealing schedule of
\citet{loshchilov2017sgdr}, which is common in modern deep learning, where the
number of training steps is fixed in advance. Each is nonincreasing with
$\eta_t \le \eta_0 \le 1/(2b)$, so all three meet the step-size condition of
Section~\ref{sec:main-results}. Table~\ref{tab:schedules} lists their
step-size sums, stated as Propositions~\ref{lem:decay-1-t},
\ref{lem:decay-1-sqrt-t} and \ref{lem:decay-cosine} in
Appendix~\ref{sec:proofs-instantiation}.

\begin{table}[!ht]
  \caption{Step-size sums for the three schedules.}
  \label{tab:schedules}
  \begin{center}
    \small
    \setlength{\tabcolsep}{5pt}
    \begin{tabular}{lccc}
      \toprule
      Schedule & $\eta_t$                                               & $\sum_{t\in[T]}\eta_t$            & $\sum_{t\in[T]}\eta_t^2$ \\
      \midrule
      $1/t$ (Prop.~\ref{lem:decay-1-t})
      & $\dfrac{\eta_0}{t}$                                    & $\eta_0(\log T+\mathcal{O}(1))$
      & $\mathcal{O}(\eta_0^2)$                                                                                               \\[6pt]
      $1/\sqrt t$ (Prop.~\ref{lem:decay-1-sqrt-t})
      & $\dfrac{\eta_0}{\sqrt t}$                              & $\eta_0(2\sqrt T+\mathcal{O}(1))$
      & $\eta_0^2(\log T+\mathcal{O}(1))$                                                                                     \\[6pt]
      Cosine (Prop.~\ref{lem:decay-cosine})
      & $\dfrac{\eta_0}{2}\left(1+\cos\dfrac{\pi t}{T}\right)$
      & $\dfrac{\eta_0}{2}(T-1)$                               & $\dfrac{\eta_0^2}{8}(3T-4)$                                  \\
      \bottomrule
    \end{tabular}
  \end{center}
\end{table}

The bound of Theorem~\ref{thm:main1} depends on the schedule only through
$\sum_{t\in[T]}\eta_t^2$, and the fourth column of Table~\ref{tab:schedules}
shows that the three schedules differ in how this sum, and hence the bound,
grows with $T$. Under $1/t$ decay the step-size sums contribute only the
constant factor $\log(e/\delta)+\eta_0^2$, so all dependence on $T$ comes
from the tail factor and the bound remains sub-polynomial in $T$ even as
$T\to\infty$, whereas under cosine decay the bound grows linearly in $T$. The
contrast stems from square-summability, since $1/t$ decay satisfies the classical
conditions $\sum_{t=1}^\infty \eta_t = \infty$, $\sum_{t=1}^\infty \eta_t^2 <
\infty$ of stochastic approximation, whereas cosine decay is designed for a
fixed horizon $T$ and its $\sum_{t\in[T]}\eta_t^2$ grows linearly in $T$. The
$1/\sqrt t$ decay lies between the two. It is not square-summable either, but
its $\sum_{t\in[T]}\eta_t^2$ grows only logarithmically.

The third column of Table~\ref{tab:schedules} determines the horizon at which
condition \eqref{eq:budget} is met. Because $\sum_{t\in[T]}\eta_t$ grows only
logarithmically under $1/t$ decay, \eqref{eq:budget} is met only at the
exponentially long horizon $T \asymp
\exp(\eta_0^{-1}\sqrt{n/(d+\log(1/\delta))})$. A slower decay removes this
restriction, because under $1/\sqrt t$ decay the cumulative step size diverges as
$\sqrt T$, so \eqref{eq:budget} is solved by the polynomial horizon $T \asymp
n/(\eta_0^2(d+\log(1/\delta)))$, and
Corollary~\ref{cor:horizon} then gives the rate $\sqrt{(d+\log(1/\delta))/n}$
up to the tail factor and the extra $\log T$.

\subsection{Global Bound under the PL Condition}
\label{sec:pl}
Theorems~\ref{thm:main1} and~\ref{thm:main3} bound gradient norms along the
trajectory. To obtain guarantees for risk values, we introduce the
Polyak--{\L}ojasiewicz (PL) condition and analyze the empirical-risk
optimization error at the last iterate, followed by the excess population
risk $\cL(\bm{w}_{T+1}) - \cL^\star$.

A differentiable function $F$ with $F^\star := \inf_{\bm{w}\in\cW}F(\bm{w}) >
-\infty$ satisfies the PL condition with $\mu>0$ if, for every
$\bm{w}\in\cW$,
\begin{equation}
  \|\nabla F(\bm{w})\|^2 \ge 2\mu\left(F(\bm{w}) - F^\star\right).
  \label{eq:pl}
\end{equation}
The condition stems from \citet{polyak1963gradient} and
\citet{lojasiewicz1963topologique}. \citet{karimi2016linear} describe its relations
to other conditions under which gradient methods converge linearly. It is
implied by $\mu$-strong convexity but does not imply convexity, and it permits a
continuum of minimizers. It excludes stationary points that are not
global minimizers. Local PL conditions have been established for wide neural
networks in a ball around initialization \citep{liu2022loss}.
Here we assume the PL condition globally on $\cW=\mathbb{R}^d$;
the local result alone does not establish this assumption.
\citet{charles2018stability} use the PL condition for stability and
generalization guarantees.

\begin{assumption}
  \label{asm:pl}
  There exist deterministic constants $\mu_0,\mu>0$, independent of $n$ and
  $d$, such that, for every sample size $n$ and every $S\in\cZ^n$, the empirical
  risk $\cL_S$ satisfies \eqref{eq:pl} on $\cW$ with constant $\mu_0$, and the
  population risk $\cL$ satisfies \eqref{eq:pl} on $\cW$ with constant $\mu$.
\end{assumption}

The next theorem bounds the optimization error in the empirical risk at the
last SGD iterate under the PL condition.

\begin{theorem}
  \label{thm:pl}
  Under the empirical-risk condition in Assumption~\ref{asm:pl}, fix
  $t_0 \ge \max\{8b/\mu_0,\,1\}$ and set
  $\eta_t := 4/(\mu_0(t+t_0))$ for $t=1,\dots,T$.
  Then, for any $\delta\in(0,1)$ and $T\ge2$, with probability $1-\delta$,
  \[
    \cL_S(\bm{w}_{T+1}) - \cL_S(\bm{w}(S))
    = \mathcal{O}\!\left(\frac{\exp\!\left(2\log^{1/\beta}(T/\delta)\right)
    \sqrt{\log T\,\log(e/\delta)}}{T}\right).
  \]
\end{theorem}

The schedule in Theorem~\ref{thm:pl} is nonincreasing and satisfies
$\eta_t \le 4/(\mu_0 t_0) \le 1/(2b)$, meeting the step-size condition of
Section~\ref{sec:main-results}. Its step-size sums have the same orders as those of
the $1/t$ schedule in Proposition~\ref{lem:decay-1-t}:
$\sum_{t\in[T]}\eta_t = \mathcal{O}(\log T)$ and
$\sum_{t\in[T]}\eta_t^2 = \mathcal{O}(1)$. Stating the results for a general
nonincreasing schedule appears technically possible, but would substantially
complicate the argument, so we do not pursue it here.

To pass from empirical to population risk, we also need to control the
discrepancy between their gradients at the last iterate. Applying
Theorem~\ref{thm:main2} with the step-size sums above gives, for any
$\delta \in (0,1)$ and any $T \ge 2$, with probability $1-\delta$,
\begin{equation}
  \max_{t\in[T+1]}\|\nabla \cL(\bm{w}_t) - \nabla \cL_S(\bm{w}_t)\|^2
  = \mathcal{O}\!\left(\frac{\exp\!\left(2\log^{1/\beta}(T/\delta)\right)
  \left(d+\log(1/\delta)\right)\log(e/\delta)\,\log T}{n}\right),
  \label{eq:pl-gradgen}
\end{equation}
where $\sum_{t\in[T]}\eta_t^2$ is absorbed into $\log(e/\delta) \ge 1$.

This bound follows from the step-size schedule and does not require the PL
condition. Combining it with Theorem~\ref{thm:pl} and the population PL
condition in Assumption~\ref{asm:pl} bounds the excess population risk at the
last SGD iterate.

\begin{theorem}
  \label{thm:pl-excess}
  Under Assumption~\ref{asm:pl}, with
  the step-size schedule in Theorem~\ref{thm:pl} and a horizon $T \asymp n$
  with $T\ge2$, for any $\delta\in(0,1)$ and $n\ge2$, with probability $1-\delta$,
  \[
    \cL(\bm{w}_{T+1}) - \cL^\star
    = \mathcal{O}\!\left(\frac{\exp\!\left(2\log^{1/\beta}(n/\delta)\right)
    \left(d+\log(1/\delta)\right)\log(e/\delta)\,\log n}{n}\right).
  \]
\end{theorem}

Suppressing fixed problem parameters and the common tail factor, the two
terms entering Theorem~\ref{thm:pl-excess} are the optimization error of
Theorem~\ref{thm:pl}, of order $\sqrt{\log T}/T$, and the gradient
generalization error \eqref{eq:pl-gradgen}, of order $\log T/n$. At
$T \asymp n$, the generalization term dominates, and its growth with $T$
limits the improvement that this bound can certify from further iterations.

\subsection{SGD with Gradient Clipping}
\label{sec:clipping}
We next extend the analysis to SGD with gradient clipping under the same
noise model.
The preceding analysis of SGD uses Assumption~\ref{asm:cap}
and the step-size restriction $\eta_t\le 1/(2b)$. By choosing a clipping level
based on the noise-tail bound, we obtain a high-probability guarantee without
these two conditions. The resulting bound controls a quantity that is
quadratic in the empirical gradient norm below the clipping level and linear
above it, while retaining the same tail factor as the bound for SGD.
Our chosen clipping level depends on $K$, $\beta$, $T$, and $\delta$, making
explicit how the noise-tail parameters affect this guarantee.

\begin{algorithm}[h]
  \caption{SGD with Gradient Clipping}
  \label{alg:clip}
  \begin{algorithmic}[1]
    \STATE {\bfseries Input:} initial point $\bm{w}_1 = \bm{0}$, step-size sequence
    $\{\eta_t\}_t$, dataset $S = \{z_1,\dots,z_n\}$, clipping level $\tau>0$.
    \FOR{$t = 1,\dots,T$}
    \STATE Draw $j_t$ uniformly from $[n]$.
    \STATE $\tilde{\bm{g}}_t \leftarrow \nabla \ell(\bm{w}_t; z_{j_t})
    \min\left\{1,\ \tau/\|\nabla \ell(\bm{w}_t; z_{j_t})\|\right\}$.
    \STATE $\bm{w}_{t+1} \leftarrow \bm{w}_t - \eta_t \tilde{\bm{g}}_t$.
    \ENDFOR
    \STATE {\bfseries Output:} iterates $\bm{w}_1,\dots,\bm{w}_{T+1}$.
  \end{algorithmic}
\end{algorithm}

To control large noise realizations over all $T$ steps, we choose a clipping
level based on the tail bound of Proposition~\ref{thm:tail}, using a per-step
exceedance probability of $\delta/(4T)$. Specifically, we set
\begin{equation}
  \tau := 8K\exp\!\left(\log^{1/\beta}(8T/\delta)\right).
  \label{eq:clip-level}
\end{equation}
The clipping level $\tau$ depends on the Orlicz norm bound $K$, the tail index
$\beta$, the horizon $T$, and the confidence parameter $\delta$.

The following theorem bounds a step-size-weighted sum that is quadratic in the
empirical gradient norm below the clipping level and linear above it.

\begin{theorem}
  \label{thm:clip}
  Let $\beta>1$, $T\ge4$ and $\delta\in(0,1)$, let $\{\eta_t\}_{t\in[T]}$ be
  positive and nonincreasing, and suppose Assumptions~\ref{asm:gradnoise} and
  \ref{asm:smooth} hold. Then Algorithm~\ref{alg:clip}, with the
  clipping level $\tau$ of \eqref{eq:clip-level}, satisfies the following bound
  with probability $1-\delta$:
  \[
    \sum_{t\in[T]} \eta_t \min\left\{\tau\|\nabla \cL_S(\bm{w}_t)\|,\
    \|\nabla \cL_S(\bm{w}_t)\|^2\right\}
    = \mathcal{O}\!\left(\exp\!\left(2\log^{1/\beta}(T/\delta)\right)
    \left(\log(e/\delta)+\sum_{t\in[T]}\eta_t^2\right)\right).
  \]
\end{theorem}

At a step with $\|\nabla \cL_S(\bm{w}_t)\|\le\tau$ the minimum is
the squared gradient norm and the summand matches that in
Theorem~\ref{thm:main1}. The other branch weakens the statement only where the
empirical gradient exceeds the clipping level. An analogous two-regime
stationarity measure, without the factor $\tau$ in the linear branch, is
used in \citet[Theorem 3.19]{li2022high}. Clipping bounds the contribution of the gradient noise by the level $\tau$
instead of by Assumption~\ref{asm:cap}, and the two powers of $\tau$ that
enter reproduce the tail factor of Theorem~\ref{thm:main1}, so both that
assumption and the step-size restriction $\eta_t\le1/(2b)$ are removed without
increasing that factor.

\section{Conclusion}
\label{sec:conclusion}

We establish high-probability optimization and generalization guarantees
for SGD under the $\beta$-heavy-tailed condition, allowing tails beyond
sub-Weibull while retaining all polynomial moments. Under smoothness and
trajectory assumptions, we obtain empirical and population stationarity
guarantees; with PL conditions, we bound last-iterate empirical optimization
error and excess population risk. For gradient-clipped SGD, a two-regime
stationarity bound removes the condition on step-size-scaled empirical
gradients and the smoothness-dependent step-size restriction.

Our noise assumption requires a uniform deterministic bound on the
conditional Orlicz scale.
Estimating or bounding this scale along SGD trajectories and determining
when uniformity is justified are concrete next steps. Tracking its dependence
on sample size and training time would help assess the resulting guarantees.
It would also be interesting to extend the analysis to scales that vary
along the trajectory without requiring a uniform deterministic bound.

The fixed-dataset setting leaves open an extension to online SGD driven by
a population stochastic oracle. Our restriction to standard and clipped
SGD also motivates extending the analysis to a broader range of optimization
methods, including SGD with momentum, Adam \citep{kingma2015adam}, and the
more recent Muon \citep{jordan2024muon}.

More broadly, understanding how architecture, data distributions, learning
dynamics, and accumulated numerical errors shape gradient noise would
clarify when the $\beta$-heavy-tailed model applies and which other tail
regimes arise. The update rule itself may alter subsequent noise
distributions through its effect on the trajectory. This understanding
could guide both the design of training
methods suited to the noise they encounter and a general Orlicz-space
framework for optimization and generalization across these regimes.

\section*{Acknowledgements}

The first author is supported by JSPS KAKENHI Grant Number JP26KJ1187.
The second author is supported by JSPS KAKENHI Grant Numbers JP25K24910,
JP25H01452, and JP24K21316.
The third author is supported by Grant-in-Aid for Research
Activity Start-up 25K23330.

\bibliography{cite_arxiv}
\bibliographystyle{plainnat}

\clearpage
\appendix
\section*{Appendix}

In the appendices we collect the technical lemmas and concentration
inequalities used throughout the paper
(Appendices~\ref{sec:general-lemmas}--\ref{sec:beta-heavy-ineq}), together with the
proofs of all theorems, lemmas, propositions, and corollaries stated in the
main text (Appendices~\ref{sec:proofs-main}--\ref{sec:proofs-clip}).

\section{General Lemmas}
\label{sec:general-lemmas}

For ease of reference, we collect here the standard lemmas used in the
proofs throughout this appendix.

\begin{lemma}[Layer-cake representation]
  \label{thm:layercake}
  Let $\Psi$ be a Young function (Definition~\ref{def:young}), and let $\Psi'$ denote its (a.e.\ defined,
  nonnegative, nondecreasing) derivative, so that
  $\Psi(x) = \int_0^x \Psi'(w)\, dw$ for all $x\ge0$. Then for any random
  variable $W$,
  \begin{equation}
    \EE[\Psi(|W|)] = \int_0^\infty \Psi'(w)\, \PP(|W|>w)\, dw.
    \label{eq:layercake}
  \end{equation}
\end{lemma}

\begin{proof}
  Writing $\Psi(|W|) = \int_0^{|W|} \Psi'(w)\, dw$ and applying Tonelli's
  theorem to the nonnegative integrand $\Psi'(w)\ind{w<|W|}$, we have
  \begin{align*}
    \EE[\Psi(|W|)] = \EE\left[\int_0^\infty \Psi'(w)\, \ind{w<|W|}\, dw\right]
    & = \int_0^\infty \Psi'(w)\, \EE[\ind{w<|W|}]\, dw \\
    & = \int_0^\infty \Psi'(w)\, \PP(|W|>w)\, dw.
  \end{align*}
\end{proof}

\begin{lemma}[Doob's $L^2$ maximal inequality; \citealp{williams1991probability}]
  \label{thm:doob}
  Let $(\cF_n)_{n=0}^N$ be a filtration and let $\{S_n\}_{n=0}^N$ be an
  $\mathbb{R}^d$-valued martingale with respect to $(\cF_n)_{n=0}^N$ with
  $S_0 = 0$. Then for any $z > 0$,
  \begin{equation}
    \PP\!\left(\max_{n\in[N]} \|S_n\| \ge z\right)
    \le \frac{\EE\left[\|S_N\|^2\right]}{z^2}.
    \label{eq:doob}
  \end{equation}
\end{lemma}

\begin{lemma}[Ville's maximal inequality;
  {\citealp[Lemma 1 and Section 6.1]{howard2020time}}]
  \label{lem:ville}
  Let $(\cF_k)_{k=0}^n$ be a filtration and let
  $\{Z_k\}_{k=0}^n$ be a nonnegative, real-valued supermartingale with respect
  to $(\cF_k)_{k=0}^n$, with $Z_0\le1$ almost surely. Then, for any $z>0$,
  \[
    \PP\!\left(\max_{k\in[n]} Z_k \ge z\right) \le \frac1z.
  \]
\end{lemma}

\begin{lemma}[Pinelis' inequality; {\citealp[Theorem 3.5]{pinelis1994optimum}}]
  \label{thm:pinelis}
  Let $\{w_i\}_{i\in[N]}$ be an $\mathbb{R}^d$-valued MDS with respect to the
  filtration $\cF_i$ satisfying $\|w_i\| \le b_i$ almost surely for
  deterministic constants $b_1,\dots,b_N$. Then for any $N \ge 1$ and $z > 0$,
  \begin{equation}
    \PP\!\left(\max_{n\in[N]} \left\|\sum_{i\in[n]} w_i\right\| \ge z\right)
    \le 2\exp\!\left(-\frac{z^2}{2\sum_{i\in[N]} b_i^2}\right).
    \label{eq:pinelis}
  \end{equation}
\end{lemma}

We also record the following standard consequence of smoothness,
used in the proof of Theorem~\ref{thm:main1}.

\begin{lemma}[{\citealp[Lemma 1.2.3]{nesterov2004introductory}}]
  \label{thm:smooth}
  Let $F:\mathbb{R}^d\to\mathbb{R}$ be differentiable.
  Suppose that, for some $b>0$,
  $\|\nabla F(\bm{w})-\nabla F(\bm{w}')\|\le b\|\bm{w}-\bm{w}'\|$
  for all $\bm{w},\bm{w}'\in\mathbb{R}^d$. Then, for all $\bm{w},\bm{w}'\in\mathbb{R}^d$,
  \[
    F(\bm{w})-F(\bm{w}') \le \inner{\bm{w}-\bm{w}'}{\nabla F(\bm{w}')} + \tfrac{b}{2} \|\bm{w}-\bm{w}'\|^2.
  \]
\end{lemma}

Under Assumption~\ref{asm:smooth}, this lemma applies to
$F=\ell(\cdot;z)$ for each fixed $z\in\cZ$, and to their average $F=\cL_S$.

The next lemma converts a function-value gap into a bound on the gradient norm,
in the direction opposite to \eqref{eq:pl}.

\begin{lemma}
  \label{lem:selfbound}
  Let $F : \mathbb{R}^d \to \mathbb{R}$ be differentiable and suppose that, for some
  $b>0$,
  $\|\nabla F(\bm{w}) - \nabla F(\bm{w}')\| \le b\|\bm{w}-\bm{w}'\|$ for all
  $\bm{w},\bm{w}'\in\mathbb{R}^d$.
  Suppose also that $F$ attains its minimum over $\mathbb{R}^d$ at
  $\bm{w}^\dagger$. Then, for every $\bm{w}\in\mathbb{R}^d$,
  \[
    \|\nabla F(\bm{w})\|^2 \le 2b\left(F(\bm{w}) - F(\bm{w}^\dagger)\right).
  \]
\end{lemma}

\begin{proof}
  Fix $\bm{w}\in\mathbb{R}^d$ and set
  $\bm{v} := \bm{w} - \tfrac1b\nabla F(\bm{w})\in\mathbb{R}^d$.
  Lemma~\ref{thm:smooth} gives
  \[
    F(\bm{v}) - F(\bm{w})
    \le \inner{\bm{v}-\bm{w}}{\nabla F(\bm{w})} + \tfrac12 b\|\bm{v}-\bm{w}\|^2
    = -\frac{1}{2b}\|\nabla F(\bm{w})\|^2,
  \]
  and $F(\bm{v}) \ge F(\bm{w}^\dagger)$ gives the claim.
\end{proof}

The following external result, used in the proof of Theorem~\ref{thm:main2},
gives a uniform generalization bound for the gradient of a nonconvex
objective.

\begin{lemma}[{\citealp[Corollary~2]{lei2021learning}}]
  \label{lm:gradgen}
  Let $\delta\in(0,1)$ and $R>0$, write
  $B_R := \{\bm{w} \in \mathbb{R}^d : \|\bm{w}\| \le R\}$ for the centered ball
  of radius $R$, and let $S = \{z_1,\dots,z_n\}$ be a collection of i.i.d.\
  samples. Suppose that Assumption~\ref{asm:smooth} holds and that
  $b' := \sup_{z\in\cZ}\|\nabla \ell(\bm{0};z)\| < \infty$. Then, with
  probability at least $1-\delta$,
  \[
    \sup_{\bm{w}\in B_R} \|\nabla \cL(\bm{w}) - \nabla \cL_S(\bm{w})\|
    \le \frac{b R + b'}{\sqrt n}
    \left(2 + 2\sqrt{48e\sqrt2(\log2 + d\log(3e))} + \sqrt{2\log(1/\delta)}\right).
  \]
\end{lemma}

The following lemma justifies the rescalings of the
confidence level $\delta$ by a constant factor performed at the end of the
proofs of our main theorems.

\begin{lemma}
  \label{lm:rescale}
  Let $\beta > 1$, $a \ge 1$ and $c > 0$. Then for every $x \ge 1$ we have
  \[
    \exp\!\left(c\log^{1/\beta}(ax)\right)
    \le \exp\!\left(c\log^{1/\beta}a\right)
    \exp\!\left(c\log^{1/\beta}x\right),
  \]
  and for every $\delta \in (0,1)$ we have
  $\log(ea/\delta) \le \left(1 + \log a\right)\log(e/\delta)$.
\end{lemma}

\begin{proof}
  The first claim follows from the subadditivity of $u\mapsto u^{1/\beta}$
  on $[0,\infty)$ and $\log(ax)=\log a+\log x$. The second follows from
  $\log(ea/\delta)=\log(e/\delta)+\log a$ and $\log(e/\delta)\ge1$.
\end{proof}

The next technical comparison is used in the truncated moment-generating
function bound below.

\begin{lemma}
  \label{lem:key}
  $\psi(s) := \log^\beta(1+s)/s$ ($s>0$) attains its unique maximum at some
  $s^*>0$, with $\log^\beta(1+s^*) < \beta^\beta$. Hence for any $m>0$ and
  $u\in[0,m]$,
  \begin{equation}
    \frac{u\log^\beta(1+m)}{m} \le \log^\beta(1+u) + \beta^\beta.
    \label{eq:key}
  \end{equation}
\end{lemma}

\begin{proof}
  The sign of $\psi'(s)$ is that of
  $g(s):=\beta s/(1+s)-\log(1+s)$. Since $g(0)=0$,
  $g'(s)=(\beta-1-s)/(1+s)^2$, and $g(s)\to-\infty$, $g$ has a unique
  positive zero $s^*$, at which $\psi$ attains its unique maximum and
  $\log(1+s^*)=\beta s^*/(1+s^*)<\beta$.
  For $s^*\le u\le m$, monotonicity gives
  $u\psi(m)\le u\psi(u)=\log^\beta(1+u)$; for $0\le u<s^*$,
  $u\psi(m)\le s^*\psi(s^*)<\beta^\beta$. These bounds imply \eqref{eq:key}.
\end{proof}

\section{Orlicz Spaces}
\label{sec:orlicz}

We briefly collect the minimal background on Orlicz spaces needed to follow
the constructions in this paper. We first recall Young functions and the
Luxemburg functional. Squaring a $\beta$-heavy-tailed variable then leads us
to the modified function $\Psi_{\beta,2^{-\beta}}$, which need not be convex.
We handle this loss of convexity by showing that $\Psi_{\beta,c}$ is a weak
$\Phi$-function and hence is equivalent to a Young function. This ultimately
yields a tail bound for sums with a comparison constant depending only on
$\beta$ and $c$. See
\citet{rao1991theory} for a comprehensive treatment (including proofs of the
facts below) and \citet{chamakh2021orlicz} for the specialization to
$\beta$-heavy tails used throughout this paper.

\begin{definition}[Young function]
  \label{def:young}
  A function $\Psi:[0,\infty)\to[0,\infty)$ is a \emph{Young function} if it
  is nondecreasing, convex, and satisfies $\Psi(0)=0$.
\end{definition}
The function $\Psi_\beta$ is a Young function for every $\beta>1$.
We next recall its standard Luxemburg norm.
\begin{definition}[Luxemburg norm and Orlicz space]
  \label{def:orlicznorm}
  Let $\Psi$ be a Young function, and let $X$ be a real-valued random variable
  on a probability space $(\Omega,\cF,\PP)$. The \emph{Luxemburg norm} of $X$
  with respect to $\Psi$ is
  \[
    \|X\|_\Psi := \inf\left\{\lambda>0 : \EE\!\left[\Psi\!\left(\frac{|X|}{\lambda}\right)\right] \le 1\right\}
  \]
  and the associated \emph{Orlicz space} $L^\Psi(\Omega,\PP)$ is the set of
  all $X$ with $\|X\|_\Psi < \infty$. (More generally this construction
    applies to measurable functions on any measure space, but we only need the
  probability-space case here.)
\end{definition}
For a function $\Psi$ that is merely nondecreasing and satisfies $\Psi(0)=0$,
we use the same formula and call the resulting quantity the
\emph{Luxemburg functional}. We need this extension for the nonconvex
functions that arise after squaring. Although the triangle inequality can
then fail, homogeneity and monotonicity remain valid; below we recover a
triangle inequality up to a constant by passing to an equivalent Young
function.

We use the following standard facts about Luxemburg functionals.

\begin{lemma}[Basic properties of the Luxemburg functional]
  \label{lem:orlicz-props}
  Let $\Psi:[0,\infty)\to[0,\infty)$ be nondecreasing with $\Psi(0)=0$, and let
  $K>0$. Properties (i), (iii) and (iv) below need no further assumption, while
  (ii) assumes in addition that $\Psi$ is convex, that is, that $\Psi$ is a
  Young function.
  \begin{enumerate}
    \item[(i)] (Homogeneity) $\|cX\|_\Psi = |c|\,\|X\|_\Psi$ for any random
      variable $X$ and scalar $c$.
    \item[(ii)] (Triangle inequality) $\|X+Y\|_\Psi \le \|X\|_\Psi +
      \|Y\|_\Psi$ for any random variables $X,Y$.
    \item[(iii)] (Equivalent formulation) If $\Psi$ is continuous, then
      $\|X\|_\Psi \le K$ if and only if $\EE[\Psi(|X|/K)]\le1$.
    \item[(iv)] (Monotonicity) If $|X| \le |Y|$ almost surely, then
      $\|X\|_\Psi \le \|Y\|_\Psi$.
  \end{enumerate}
\end{lemma}
The proofs of Properties (i)--(iii) can be found in \citet{rao1991theory}.
Property (iv) is immediate from the definition of the Orlicz norm, since
$\Psi$ is nondecreasing on $[0,\infty)$, so that any $c>0$ with
$\EE[\Psi(|Y|/c)]\le1$ also satisfies $\EE[\Psi(|X|/c)]\le1$.

Property (iii), applied to $\Psi=\Psi_\beta$, is exactly the equivalence
between the two formulations of $\beta$-heavy tailedness recorded in
Definition~\ref{def:betaheavy}.

To control squared noise, we introduce a coefficient $c>0$ in the exponent
of $\Psi_\beta$.

\begin{definition}
  We further parametrize the constant inside the exponent of $\Psi_\beta$ by
  $c>0$, defining
  \begin{equation}
    \Psi_{\beta,c}(x) := \exp(c \cdot \log^\beta(x+1)) - 1.
    \label{eq:families}
  \end{equation}
  When $c=1$ we abbreviate
  $\Psi_\beta = \Psi_{\beta,1}$. We similarly write
  \[
    \overline{\Psi}_{\beta,c}(x) := \exp(-c\log^\beta(x+1)) = 1/(\Psi_{\beta,c}(x)+1),
  \]
  so that $\overline{\Psi}_\beta = \overline{\Psi}_{\beta,1}$ recovers the
  notation of Section~\ref{sec:notation}.
\end{definition}

Every $\Psi_{\beta,c}$ is continuous and increasing, vanishes at the origin,
and tends to infinity. The following proposition explains why the parameter
$c$ is needed: the square of a $\beta$-heavy-tailed random variable remains
$\beta$-heavy tailed after the coefficient in the exponent is reduced to at
most $2^{-\beta}$.

\begin{proposition}
  \label{thm:square}
  Let $\beta>1$ and let $X$ be a $\beta$-heavy-tailed random variable. Then
  \[
    \|X^2\|_{\Psi_{\beta,\lambda}} \le \|X\|_{\Psi_\beta}^2
    \qquad \text{for every } \lambda\in(0,2^{-\beta}].
  \]
\end{proposition}

\begin{proof}
  For $u\ge0$ and $0<\lambda\le2^{-\beta}$, the inequality
  $1+u^2\le(1+u)^2$ gives $\Psi_{\beta,\lambda}(u^2)\le\Psi_\beta(u)$.
  Thus, for every $K>\|X\|_{\Psi_\beta}$,
  \[
    \EE\!\left[\Psi_{\beta,\lambda}(X^2/K^2)\right]
    \le \EE\!\left[\Psi_\beta(|X|/K)\right]\le1.
  \]
  Lemma~\ref{lem:orlicz-props}(iii) yields
  $\|X^2\|_{\Psi_{\beta,\lambda}}\le K^2$; letting
  $K\downarrow\|X\|_{\Psi_\beta}$ proves the claim.
\end{proof}

\begin{remark}[Convexity after squaring]
  \label{rem:squared-convexity}
  For $\beta>1$, writing $s=\log(1+x)$, the sign of $\Psi_{\beta,c}''(x)$ is
  the sign of $c\beta s^{\beta-1}+(\beta-1)/s-1$, whose minimum is
  $\beta(c\beta)^{1/\beta}-1$. Thus $\Psi_{\beta,c}$ is convex if and only if
  $c\ge\beta^{-(\beta+1)}$. The function $\Psi_\beta$ is therefore convex,
  while $\Psi_{\beta,2^{-\beta}}$ is convex only for $\beta\ge\beta^\star$,
  where $\beta^\star=1.51887\ldots$ solves
  $(\beta+1)\log\beta=\beta\log2$.
  For $1<\beta<\beta^\star$, it is nonconvex in an intermediate region but
  convex near zero and for sufficiently large $x$. By comparison, the
  sub-Weibull function $\psi_\theta(x)=\exp(x^{1/\theta})-1$ is convex for
  $0<\theta\le1$ and concave on $(0,(\theta-1)^\theta)$ for $\theta>1$.
  Squaring replaces $\theta$ by $2\theta$, so in both families squaring can
  require a nonconvex generating function even when the original one is convex.
\end{remark}

This loss of convexity matters because the usual triangle inequality for the
Luxemburg norm requires a Young function. Instead of restricting $\beta$, we
use the fact that a nonconvex function can still be equivalent to a Young
function if it is a weak $\Phi$-function.

\begin{definition}[Weak $\Phi$-function]
  \label{def:weak-phi}
  Following \citet[Definitions 2.1.2 and 2.1.3]{harjulehto2019orlicz}, an
  increasing function $\Psi:[0,\infty)\to[0,\infty)$ with $\Psi(0)=0$,
  $\lim_{x\downarrow0}\Psi(x)=0$, and $\lim_{x\to\infty}\Psi(x)=\infty$ is a
  \emph{weak $\Phi$-function} if $x\mapsto\Psi(x)/x$ is almost increasing on
  $(0,\infty)$. That is, there exists $A\ge1$ such that
  \[
    \frac{\Psi(u)}{u}\le A\frac{\Psi(v)}{v}
    \qquad\text{whenever }0<u\le v.
  \]
  In the terminology of the cited reference, this last condition is
  $(\mathrm{aInc})_1$; ordinary monotonicity is the special case $A=1$.
\end{definition}

We now verify this weaker condition for $\Psi_{\beta,c}$ and invoke the
equivalence theorem to restore the triangle inequality up to a comparison
constant depending only on $\beta$ and $c$.

\begin{lemma}
  \label{lem:convexify}
  Let $\beta>1$ and $c>0$. There exist a Young function $\Phi_{\beta,c}$ and
  a finite constant $g_{\beta,c}\ge1$, depending only on $\beta,c$, such that
  $\Phi_{\beta,c}(x)\le\Psi_{\beta,c}(x)\le\Phi_{\beta,c}(g_{\beta,c}x)$ for all
  $x\ge0$. Consequently, for every random variable $X$,
  $\|X\|_{\Phi_{\beta,c}}\le\|X\|_{\Psi_{\beta,c}}\le g_{\beta,c}\|X\|_{\Phi_{\beta,c}}$,
  and for every $n\ge1$ and random variables $X_1,\dots,X_n$,
  \begin{equation}
    \left\|\sum_{i\in[n]} X_i\right\|_{\Psi_{\beta,c}}
    \le g_{\beta,c}\sum_{i\in[n]}\|X_i\|_{\Psi_{\beta,c}}.
    \label{eq:quasi-triangle}
  \end{equation}
  If $\Psi_{\beta,c}$ is convex,
  one may take $\Phi_{\beta,c}=\Psi_{\beta,c}$ and $g_{\beta,c}=1$.
\end{lemma}

\begin{proof}
  Write $\Psi=\Psi_{\beta,c}$. By
  \citet[Theorem 2.2.3]{harjulehto2019orlicz}, every weak $\Phi$-function is
  equivalent to a strong $\Phi$-function. Thus it suffices first to verify
  that $\Psi$ is a weak $\Phi$-function in the sense of
  Definition~\ref{def:weak-phi}. As observed above, $\Psi$ is increasing and
  continuous, $\Psi(0)=0$,
  $\Psi(x)\to0$ as $x\downarrow0$, and $\Psi(x)\to\infty$ as $x\to\infty$.
  It remains to check that $m(x):=\Psi(x)/x$ is almost increasing. Direct
  calculation gives
  \[
    \frac{x\Psi'(x)}{\Psi(x)}
    = \frac{x}{1+x}\,
    \frac{c\beta\log^{\beta-1}(1+x)}{1-\exp(-c\log^\beta(1+x))}
    \longrightarrow
    \begin{cases}
      \beta  & \text{as } x\downarrow0, \\
      \infty & \text{as } x\to\infty.
    \end{cases}
  \]
  Indeed,
  \[
    m'(x)=\frac{\Psi(x)}{x^2}
    \left(\frac{x\Psi'(x)}{\Psi(x)}-1\right).
  \]
  Since $\beta>1$, the displayed limits show that $m'(x)>0$ near zero and for
  all sufficiently large $x$. Thus $m$ is increasing on both tails. On the
  compact interval between these tails, $m$ is positive and continuous, so
  its maximum-to-minimum ratio is finite. Combining the three regions gives
  a constant $A\ge1$ such that $m(u)\le A m(v)$ whenever $0<u\le v$.
  Hence $m$ is almost increasing, and $\Psi$ is a weak $\Phi$-function.

  The cited theorem therefore provides a strong $\Phi$-function $\Phi_0$ (in
  particular, a Young function in the sense of Definition~\ref{def:young})
  and $L\ge1$ such that
  $\Phi_0(x/L)\le\Psi(x)\le\Phi_0(Lx)$. Taking
  $\Phi_{\beta,c}(x):=\Phi_0(x/L)$ and $g_{\beta,c}:=L^2$ gives the claimed
  comparison. The definition of the Luxemburg functional then gives the
  functional bounds, and
  \[
    \left\|\sum_{i\in[n]}X_i\right\|_{\Psi_{\beta,c}}
    \le g_{\beta,c}\left\|\sum_{i\in[n]}X_i\right\|_{\Phi_{\beta,c}}
    \le g_{\beta,c}\sum_{i\in[n]}\|X_i\|_{\Phi_{\beta,c}}
    \le g_{\beta,c}\sum_{i\in[n]}\|X_i\|_{\Psi_{\beta,c}}.
  \]
\end{proof}

Having recovered control of sums, we record the elementary one-variable tail
bound. This estimate uses only the monotonicity of $\Psi_{\beta,c}$ and
therefore remains valid even when $\Psi_{\beta,c}$ is nonconvex.

\begin{proposition}
  \label{thm:tail}
  Let $\beta>1$ and $c>0$. For a random variable $X$, we have
  \[
    \PP(|X|>t) \le 2\overline{\Psi}_{\beta,c}\!\left(\frac{t}{\|X\|_{\Psi_{\beta,c}}}\right)
    = 2\exp\!\left(-c\log^\beta\!\left(\frac{t}{\|X\|_{\Psi_{\beta,c}}}+1\right)\right).
  \]
\end{proposition}

\begin{proof}
  Write $\kappa := \|X\|_{\Psi_{\beta,c}}$. $\EE[\Psi_{\beta,c}(|X|/\kappa)+1] \le 2$, and
  $\Psi_{\beta,c}$ is increasing on $[0,\infty)$, so Markov's inequality gives
  \begin{align*}
    \PP(|X|>t) & = \PP\!\left(\Psi_{\beta,c}\!\left(\frac{|X|}{\kappa}\right)+1
    > \Psi_{\beta,c}\!\left(\frac{t}{\kappa}\right)+1\right)                    \\
    & \le \frac{2}{\Psi_{\beta,c}(t/\kappa)+1}
    = 2\overline{\Psi}_{\beta,c}\!\left(\frac{t}{\kappa}\right).
  \end{align*}
\end{proof}

Combining this tail estimate with the quasi-triangle inequality from
Lemma~\ref{lem:convexify} gives the desired bound for sums. Crucially, the
comparison constant below depends only on $\beta$ and $c$, not on the number
or distributions of the summands.

\begin{corollary}
  \label{thm:sum-beta-heavy}
  Let $\beta>1$, $c>0$, and let $X_1,\dots,X_n$ be random variables with
  $\|X_i\|_{\Psi_{\beta,c}}\le K_i$ for constants $K_i>0$. Then for all $x\ge0$,
  \begin{equation}
    \PP\!\left(\left|\sum_{i\in[n]} X_i\right| \ge x\right)
    \le 2\overline{\Psi}_{\beta,c}\!\left(\frac{x}{g_{\beta,c}\sum_{i\in[n]} K_i}\right).
    \label{eq:main}
  \end{equation}
\end{corollary}

\begin{proof}
  By \eqref{eq:quasi-triangle},
  $\|\sum_{i\in[n]}X_i\|_{\Psi_{\beta,c}}\le g_{\beta,c}\sum_{i\in[n]}K_i$.
  Markov's inequality applied to
  $1+\Psi_{\beta,c}(|\sum_{i\in[n]}X_i|/(g_{\beta,c}\sum_{i\in[n]}K_i))$ gives the claim.
\end{proof}

We fix $\Phi_{\beta,c}$ and $g_{\beta,c}$ as in Lemma~\ref{lem:convexify}, and
write $\Phi_\beta:=\Phi_{\beta,2^{-\beta}}$ and $g_\beta:=g_{\beta,2^{-\beta}}$.
In particular, $g_\beta$ is finite for every $\beta>1$ and depends only on $\beta$.

We record the form in which Corollary~\ref{thm:sum-beta-heavy} is used in the
proof of Theorem~\ref{thm:main1}. Solving
$2\overline{\Psi}_{\beta,c}(x/(g_{\beta,c}\sum_{i\in[n]} K_i)) = \delta$ for $x$ gives,
with probability at least $1-\delta$,
\begin{equation}
  \left|\sum_{i\in[n]} X_i\right|
  \le g_{\beta,c}\left(\exp\!\left(\left(\tfrac1c\log(2/\delta)\right)^{1/\beta}\right)-1\right)
  \sum_{i\in[n]} K_i .
  \label{eq:sum-hp}
\end{equation}
In particular, for $c = 2^{-\beta}$ the exponent becomes
$(2^\beta\log(2/\delta))^{1/\beta} = 2\log^{1/\beta}(2/\delta)$, so the
prefactor is $g_\beta(\exp(2\log^{1/\beta}(2/\delta))-1)$. This factor $2$,
inherited from the weakening of the Young function from $\Psi_\beta$ to
$\Psi_{\beta,2^{-\beta}}$ when squaring in Proposition~\ref{thm:square}, is the source of the constant $2$ in
the exponent of Theorem~\ref{thm:main1}. The constant $g_\beta$ depends only on
$\beta$.
\section{Inequalities for Sequences of $\beta$-Heavy-Tailed Random Variables}
\label{sec:beta-heavy-ineq}

We next give a concentration inequality for $\beta$-heavy-tailed martingale
difference sequences in $\mathbb{R}^d$, which is the form in which the
gradient noise enters. It shows that
$\PP(\max_{n\in[N]} \|\sum_{i\in[n]} X_i\| \ge z)$ can be bounded by a
sub-Gaussian term and a $\beta$-heavy tail term with an arbitrary threshold
$M$.

\begin{proposition}[$\beta$-heavy tail MDS concentration inequality]
  \label{thm:max}
  Let $\beta>1$ and let $\{X_i\}_{i\in[N]}$ be an $\mathbb{R}^d$-valued
  $(\cF_i)$-MDS. Suppose
  $\sigma_i := \left\|\,\|X_i\|\,\right\|_{\Psi_\beta} < \infty$ for all
  $i\in[N]$, and suppose $M>0$ satisfies
  \begin{equation}
    \beta\log^{\beta-1}(M+1) -2 > 0.
    \label{eq:Mcond}
  \end{equation}
  Then for any $z>0$,
  \begin{align}
    \PP\!\left(\max_{n\in[N]} \left\|\sum_{i\in[n]} X_i\right\| \ge z\right)
    \le & 2\exp\!\left(-\frac{z^2}{32M^2\sum_{i\in[N]} \sigma_i^2}\right) \notag \\
    & + \left(2M^2 + \frac{4(M+1)^2}{\beta\log^{\beta-1}(M+1)-2}\right)
    \frac{4\sum_{i\in[N]} \sigma_i^2}{z^2}\overline{\Psi}_\beta(M).
    \label{eq:master}
  \end{align}
\end{proposition}
The following corollary, obtained from Proposition~\ref{thm:max}, gives a high-probability bound for the $\beta$-heavy-tailed MDS
only dominated by $\overline{\Psi}_\beta$, in a form more convenient for use.

\begin{corollary}
  \label{cor:delta}
  Let $\beta>1$, let $\{X_i\}_{i\in[N]}$ be an $\mathbb{R}^d$-valued MDS with
  $\sigma_i = \left\|\,\|X_i\|\,\right\|_{\Psi_\beta} < \infty$, and set
  \begin{equation}
    m_\beta := \exp\!\left(\left(\tfrac3\beta\right)^{\frac{1}{\beta-1}}\right)-1,
    \quad
    D_\beta := 2 + 4\left(1+\tfrac1{m_\beta}\right)^2,
    \quad
    C_\beta := \max\{1,m_\beta\}\max\!\left\{8,\ 2\sqrt{D_\beta/\log2}\right\}.
    \label{eq:Cbeta}
  \end{equation}
  Then for any $\delta \in (0,1]$, with probability at least $1-\delta$,
  \begin{equation}
    \max_{n\in[N]} \left\|\sum_{i\in[n]} X_i\right\|
    \le C_\beta\sqrt{\log(2/\delta) \sum_{i\in[N]} \sigma_i^2} \cdot \Psi_\beta^{-1}(2/\delta).
    \label{eq:delta}
  \end{equation}
\end{corollary}

The threshold $M$ of Proposition~\ref{thm:max} is placed at the quantile of
the $\beta$-heavy tail at level $\delta$, and is raised to $m_\beta$ whenever
that quantile is too small for condition \eqref{eq:Mcond}; the level $m_\beta$
is the one at which $\beta\log^{\beta-1}(m_\beta+1) = 3$. The constant
$C_\beta$ is finite for every $\beta>1$ and depends only on $\beta$, through
$m_\beta$.

We first establish the following
technical integral estimate.

\begin{lemma}
  \label{thm:integral}
  Let $\beta>1$ and suppose $M>0$ satisfies \eqref{eq:Mcond}. Then
  \begin{equation}
    \int_M^\infty t\,\overline{\Psi}_\beta(t)\, dt
    \le \frac{(M+1)^2\,\overline{\Psi}_\beta(M)}{\beta\log^{\beta-1}(M+1)-2}.
    \label{eq:integral}
  \end{equation}
\end{lemma}

\begin{proof}
  Substituting $t = \exp(s)-1$, we have $t < \exp(s)$ and
  $\overline{\Psi}_\beta(t) = \exp(-s^\beta)$, so
  \[
    \int_M^\infty t\,\overline{\Psi}_\beta(t)\, dt
    \le \int_{\log(M+1)}^\infty \exp(2s - s^\beta)\, ds.
  \]
  We bound the right-hand side using Laplace's method. Let $\phi(s) :=
  s^\beta - 2s$. Since $\beta>1$, $\phi'(s) = \beta s^{\beta-1} - 2$ is
  monotonically increasing in $s$, and by \eqref{eq:Mcond},
  $\phi'(\log(M+1))>0$. Hence for $s \ge \log(M+1)$, $\phi'(s) \ge
  \phi'(\log(M+1)) > 0$ and $\lim_{s\to\infty}\phi(s) = \infty$, so
  \begin{align*}
    \int_{\log(M+1)}^\infty \exp(-\phi(s))\, ds
    & = \int_{\log(M+1)}^\infty \frac{\phi'(s)\exp(-\phi(s))}{\phi'(s)}\, ds             \\
    & \le \frac{1}{\phi'(\log(M+1))} \int_{\log(M+1)}^\infty \phi'(s)\exp(-\phi(s))\, ds
    = \frac{\exp(-\phi(\log(M+1)))}{\phi'(\log(M+1))}.
  \end{align*}
  Finally, since
  $\exp(-\phi(\log(M+1))) = \exp(2\log(M+1))\exp(-\log^\beta(M+1)) = (M+1)^2\overline{\Psi}_\beta(M)$
  and $\phi'(\log(M+1)) = \beta\log^{\beta-1}(M+1)-2$, the claim follows.
\end{proof}
We now prove Proposition~\ref{thm:max}.
\begin{proof}[Proof of Proposition~\ref{thm:max}]
  Fix $M>0$ satisfying \eqref{eq:Mcond}, and for each $i\in[N]$ decompose
  \begin{equation}
    \begin{aligned}
      X'_i  & := X_i \ind{\|X_i\|\le M\sigma_i} - \EE[X_i \ind{\|X_i\|\le M\sigma_i} \mid \cF_{i-1}], \\
      X''_i & := X_i \ind{\|X_i\|> M\sigma_i} - \EE[X_i \ind{\|X_i\|>M\sigma_i} \mid \cF_{i-1}].
    \end{aligned}
    \label{eq:splitM}
  \end{equation}
  Both are $\cF_i$-measurable with conditional mean zero, so $\{X'_i\}$ and
  $\{X''_i\}$ are both $\mathbb{R}^d$-valued $(\cF_i)$-MDS. Since $X_i = X'_i +
  X''_i$, the triangle inequality gives
  \begin{equation}
    \PP\!\left(\max_{n\in[N]}\left\|\sum_{i\in[n]} X_i\right\| \ge 2z\right)
    \le \PP\!\left(\max_{n\in[N]}\left\|\sum_{i\in[n]} X'_i\right\| \ge z\right)
    + \PP\!\left(\max_{n\in[N]}\left\|\sum_{i\in[n]} X''_i\right\| \ge z\right),
    \label{eq:union}
  \end{equation}
  so it suffices to bound each term on the right-hand side.

  For the first term, \eqref{eq:splitM} gives $\|X'_i\| \le 2M\sigma_i$ almost
  surely, so Lemma~\ref{thm:pinelis} with $b_i = 2M\sigma_i$ yields
  \begin{equation}
    \PP\!\left(\max_{n\in[N]}\left\|\sum_{i\in[n]} X'_i\right\| \ge z\right)
    \le 2\exp\!\left(-\frac{z^2}{8M^2\sum_{i\in[N]} \sigma_i^2}\right).
    \label{eq:part1}
  \end{equation}

  For the second term, $\{X''_i\}$ is an MDS, so $S''_n := \sum_{i\in[n]}
  X''_i$ is a martingale with $S''_0 = 0$, and Lemma~\ref{thm:doob} together
  with the orthogonality of an MDS gives
  \begin{equation}
    \PP\!\left(\max_{n\in[N]} \|S''_n\| \ge z\right)
    \le \frac{1}{z^2}\EE\left[\|S''_N\|^2\right]
    = \frac{1}{z^2}\sum_{i\in[N]}\EE\left[\|X''_i\|^2\right].
    \label{eq:part2}
  \end{equation}

  We now bound $\EE[\|X''_i\|^2]$. In general, for an $\mathbb{R}^d$-valued
  random variable $W$ and a sub-$\sigma$-algebra $\cG$ we have
  $\EE[\|W-\EE[W\mid\cG]\|^2] = \EE[\|W\|^2] - \EE[\|\EE[W\mid\cG]\|^2] \le
  \EE[\|W\|^2]$, so by Lemma~\ref{thm:layercake} applied to the nonnegative
  random variable $\|X_i\|\ind{\|X_i\|>M\sigma_i}$,
  \begin{align}
    \EE\left[\|X''_i\|^2\right] & \le \EE[\|X_i\|^2 \ind{\|X_i\|>M\sigma_i}]
    = \int_0^\infty 2y\, \PP(\|X_i\|\ind{\|X_i\|>M\sigma_i}>y)\, dy \notag                                                         \\
    & = M^2\sigma_i^2 \PP(\|X_i\|>M\sigma_i) + \int_{M\sigma_i}^\infty 2y\,\PP(\|X_i\|>y)\, dy \notag  \\
    & = M^2\sigma_i^2\PP(\|X_i\|>M\sigma_i) + \sigma_i^2 \int_M^\infty 2t\,\PP(\|X_i\|>t\sigma_i)\, dt
    \label{eq:second}
  \end{align}
  (the last equality using the substitution $y = t\sigma_i$). Applying
  Proposition~\ref{thm:tail} to both terms,
  \[
    \EE\left[\|X''_i\|^2\right] \le \sigma_i^2\left(2M^2\overline{\Psi}_\beta(M) +
    4\int_M^\infty t\overline{\Psi}_\beta(t)\, dt\right),
  \]
  and by Lemma~\ref{thm:integral},
  \begin{equation}
    \EE\left[\|X''_i\|^2\right] \le \left(2M^2 + \frac{4(M+1)^2}{\beta\log^{\beta-1}(M+1)-2}\right)
    \sigma_i^2\,\overline{\Psi}_\beta(M).
    \label{eq:secondfinal}
  \end{equation}
  Substituting \eqref{eq:part1}, \eqref{eq:part2}, \eqref{eq:secondfinal}
  into \eqref{eq:union} yields the claim.
\end{proof}

\begin{proof}[Proof of Corollary~\ref{cor:delta}]
  Write $V := \sum_{i\in[N]}\sigma_i^2$, fix $\delta \in (0,1]$, and set
  \[
    M_\delta := \exp\!\left(\log^{1/\beta}(2/\delta)\right)-1, \qquad
    \widetilde M_\delta := \max\left\{M_\delta,\, m_\beta\right\},
  \]
  so that $\overline{\Psi}_\beta(M_\delta) = \delta/2$. We check the hypotheses of Proposition~\ref{thm:max} at $M = \widetilde
  M_\delta$. By \eqref{eq:Cbeta} we have $\log(m_\beta+1) =
  (3/\beta)^{1/(\beta-1)}$ and hence $\beta\log^{\beta-1}(m_\beta+1) = 3$, so
  the map $M \mapsto \beta\log^{\beta-1}(M+1)-2$, which is increasing, is at
  least $1$ on $[m_\beta,\infty)$, and $\widetilde M_\delta$ satisfies
  \eqref{eq:Mcond}. The same monotonicity, together with the fact that
  $(1+1/M)^2$ is decreasing, shows that
  \[
    \frac{1}{M^2}\left(2M^2 + \frac{4(M+1)^2}{\beta\log^{\beta-1}(M+1)-2}\right)
    = 2 + \frac{4\left(1+1/M\right)^2}{\beta\log^{\beta-1}(M+1)-2}
  \]
  is nonincreasing on $[m_\beta,\infty)$, where it takes the value $D_\beta$ at
  $M = m_\beta$. The constant in the second term of \eqref{eq:master} is
  therefore at most $D_\beta \widetilde M_\delta^2$.

  We now apply \eqref{eq:master} with $M = \widetilde M_\delta$ and
  $z := c\,\widetilde M_\delta\sqrt{V\log(2/\delta)}$, where $c := \max\{8,
  2\sqrt{D_\beta/\log2}\}$. Since $c^2 \ge 64$ and $\delta/2 \le 1/2$, the
  first term is
  \[
    2\exp\!\left(-\frac{z^2}{32\widetilde M_\delta^2 V}\right)
    = 2\left(\frac\delta2\right)^{c^2/32}
    \le 2\left(\frac\delta2\right)^{2} \le \frac\delta2 ,
  \]
  while the second term, using $\widetilde M_\delta \ge M_\delta$ and the
  monotonicity of $\overline{\Psi}_\beta$, is at most
  \[
    D_\beta\widetilde M_\delta^2 \cdot \frac{4V}{z^2}\,\overline{\Psi}_\beta(M_\delta)
    = \frac{4D_\beta}{c^2\log(2/\delta)} \cdot \frac\delta2
    \le \frac{2D_\beta\,\delta}{c^2\log2}
    \le \frac\delta2 ,
  \]
  so that $\PP(\max_{n\in[N]}\|\sum_{i\in[n]} X_i\| \ge z) \le \delta$.

  It remains to rewrite $z$. Since $\beta>1$ and $\delta\le1$ we have
  $\log^{1/\beta}(2/\delta) \ge (\log2)^{1/\beta} > \log 2$, whence $M_\delta > 1$ and
  $\widetilde M_\delta \le \max\{1,m_\beta\}M_\delta$. Since
  $\overline{\Psi}_\beta = 1/(\Psi_\beta+1)$, we have
  $\Psi_\beta(M_\delta) = 2/\delta-1$, so monotonicity of $\Psi_\beta$ gives
  $M_\delta \le \Psi_\beta^{-1}(2/\delta)$, and altogether
  \[
    z \le C_\beta\sqrt{V\log(2/\delta)}\;\Psi_\beta^{-1}(2/\delta),
  \]
  which is the bound \eqref{eq:delta}.
\end{proof}

Proposition~\ref{thm:max} and Corollary~\ref{cor:delta} require deterministic
Orlicz norm bounds $\sigma_i$, which is too restrictive for SGD, because the increment
$\xi_t = -2\eta_t\inner{\epsilon_t}{\nabla \cL_S(\bm{w}_t)}$ has conditional
Orlicz norm bounded by the trajectory-dependent quantity
$2\eta_t\|\nabla \cL_S(\bm{w}_t)\|K$, even though the bound $K$ in
Assumption~\ref{asm:gradnoise} is deterministic.
We therefore need a Freedman-type inequality \citep{freedman1975tail} that
allows predictable bounds $K_{i-1}$ on the conditional Orlicz norms, subject
to deterministic upper bounds, and controls
the tail of the martingale sum through the accumulated variance proxy
$V_k = \sum_{i\in[k]} a_\beta K_{i-1}^2$. Part (i) of
Proposition~\ref{thm:freedman} moreover permits the self-bounding condition
$V_k \le \alpha S_k + \gamma$, the form used to prove
Theorem~\ref{thm:main1}. Throughout, the Orlicz condition (A3) below is read
with the conventions $u/0 := +\infty$ for $u>0$ and $0/0 := 0$, so that on the
event $\{K_{i-1}=0\}$ it forces $\xi_i=0$ almost surely, since any other value makes
the left-hand side infinite. The bound $K_{i-1}$ is thus allowed to
vanish, in which case the corresponding increment vanishes as well.

\begin{proposition}[$\beta$-heavy tail Freedman inequality]
  \label{thm:freedman}
  Let $(\Omega,\cF,(\cF_i)_{i\ge0},\PP)$ be a filtered probability space and
  $N\in\mathbb{N}$, and suppose the $(\cF_i)$-adapted real random variable
  sequences $(\xi_i)_{i\in[N]}$, $(K_i)_{i\ge0}$ satisfy the following almost
  surely for each $i\in[N]$:
  \begin{align*}
    \text{(A1)} \quad & 0 \le K_{i-1} \le m_i \quad (m_i \text{ a deterministic constant}), \quad
    \mbar := \max_{i\in[N]} m_i;                                                                               \\
    \text{(A2)} \quad & \EE[\xi_i \mid \cF_{i-1}] = 0;                                                         \\
    \text{(A3)} \quad & \EE\left[\Psi_\beta\!\left(\frac{|\xi_i|}{K_{i-1}}\right) \mid \cF_{i-1}\right] \le 1.
  \end{align*}
  Fix $\delta\in(0,1)$ and set $\Lambda := \log(N/\delta)$. For notational simplicity we define
  \begin{equation}
    \begin{aligned}
      J_\beta & := \int_0^\infty \exp(2s - s^\beta/2)\, ds, \qquad
      a_\beta := 8\exp(\beta^\beta/2) J_\beta,                     \\
      M       & := \mbar(\exp(\Lambda^{1/\beta})-1), \qquad
      \bF := \frac{2(\exp(\Lambda^{1/\beta})-1)}{\Lambda}.
    \end{aligned}
    \label{eq:freedef}
  \end{equation}
  Let $S_k := \sum_{i\in[k]} \xi_i$ and define the accumulated variance proxy
  $V_k := \sum_{i\in[k]} a_\beta
  K_{i-1}^2$. Then the following hold.

  (i) For any $x,\gamma\ge0$, $\alpha\ge \bF\mbar$, $\lambda\in[0,1/(2\alpha)]$,
  \begin{equation}
    \PP\!\left(\bigcup_{k\in[N]} \{S_k \ge x \text{ and } V_k \le \alpha S_k + \gamma\}\right)
    \le \exp(-\lambda x + 2\lambda^2\gamma) + 2\delta.
    \label{eq:main1}
  \end{equation}

  (ii) For any $x,\gamma\ge0$, $\lambda\in[0,1/(\bF\mbar)]$,
  \begin{equation}
    \PP\!\left(\bigcup_{k\in[N]} \{S_k \ge x \text{ and } V_k \le \gamma\}\right)
    \le \exp(-\lambda x + (\lambda^2/2)\gamma) + 2\delta.
    \label{eq:main2}
  \end{equation}
\end{proposition}

The event $\{S_k \ge x \text{ and } V_k \le \alpha S_k + \gamma\}$ can be
interpreted as the event that the MDS sum $S_k$ exceeds $x$ while its
accumulated variance proxy $V_k$ remains controlled.

We first prove several lemmas needed for the proof of Proposition~\ref{thm:freedman}.

\begin{lemma}[Conditional tail estimate of the variance]
  \label{lem:tail}
  For each $i\in[N]$ and any $t\ge0$,
  \begin{equation}
    \PP(|\xi_i|>t \mid \cF_{i-1}) \le 2\overline{\Psi}_\beta\!\left(\frac{t}{K_{i-1}}\right),
    \label{eq:tail}
  \end{equation}
  \begin{equation}
    \EE[\xi_i^2 \mid \cF_{i-1}] \le 4J_\beta K_{i-1}^2.
    \label{eq:var}
  \end{equation}
\end{lemma}

\begin{proof}
  If $K_{i-1}=0$, then $\xi_i=0$ almost surely by (A3), and both claims are
  immediate, because the right-hand side of \eqref{eq:tail} vanishes for $t>0$ and
  equals $2$ for $t=0$, while both sides of \eqref{eq:var} vanish. We may
  therefore assume $K_{i-1}>0$ in what follows.

  The bound \eqref{eq:tail} follows in the same way as Proposition~\ref{thm:tail}.
  For \eqref{eq:var}, note first that $J_\beta<\infty$: since $\beta>1$, we
  have $s^\beta/2 - 2s \to \infty$ as $s\to\infty$, so the integrand of
  $J_\beta$ is continuous on $[0,\infty)$ and eventually dominated by
  $e^{-s}$. Now \eqref{eq:var} follows from Lemma~\ref{thm:layercake} and
  \eqref{eq:tail}, substituting $t=K_{i-1}u$ and then $1+u=e^s$:
  \begin{align*}
    \EE[\xi_i^2\mid\cF_{i-1}] = \int_0^\infty 2t\, \PP(|\xi_i|>t\mid\cF_{i-1})\, dt
    & \le 4K_{i-1}^2 \int_0^\infty u\,\overline{\Psi}_\beta(u)\, du \\
    & = 4K_{i-1}^2 \int_0^\infty (e^s-1)e^s\exp(-s^\beta)\, ds
    \le 4J_\beta K_{i-1}^2,
  \end{align*}
  where the last inequality uses $(e^s-1)e^s \le e^{2s}$ and $s^\beta/2 \le
  s^\beta$ for $s\ge0$.
\end{proof}

Since a truncation technique is used later in the proof, we give a moment
generating function (MGF) estimate for a random variable that is bounded
above.

\begin{lemma}[MGF estimate for an upper-bounded random variable]
  \label{lem:mgf}
  Fix $i\in[N]$. For $M>0$ and
  \begin{equation}
    0 \le \lambda \le \frac{1}{2M}\log^\beta\!\left(1+\frac{M}{K_{i-1}}\right),
    \label{eq:lambbound}
  \end{equation}
  we have
  \begin{equation}
    \EE[\exp(\lambda\xi_i)\ind{\xi_i\le M} \mid \cF_{i-1}]
    \le \exp\!\left(\frac{\lambda^2 a_\beta K_{i-1}^2}{2}\right).
    \label{eq:mgf}
  \end{equation}
\end{lemma}

\begin{proof}
  If $K_{i-1}=0$, then $\xi_i=0$ almost surely by (A3), the constraint
  \eqref{eq:lambbound} is vacuous since its right-hand side is infinite, and
  \eqref{eq:mgf} reduces to the identity $\EE[\ind{0\le M}\mid\cF_{i-1}] = 1 =
  \exp(0)$. We may therefore assume $K_{i-1}>0$, which in particular justifies the
  substitution $t=K_{i-1} u$ used below.

  For brevity write $\EE_{i-1}[\cdot] := \EE[\cdot\mid\cF_{i-1}]$. Let
  $\phi(u) := e^u - 1 - u$, so that $e^{\lambda z} = 1+\lambda z +
  \phi(\lambda z)$. Since $\EE_{i-1}[\xi_i]=0$ and $M>0$,
  \[
    \EE_{i-1}[\xi_i\ind{\xi_i\le M}] = -\EE_{i-1}[\xi_i\ind{\xi_i>M}] \le 0,
  \]
  and hence
  \begin{align}
    \EE_{i-1}[\exp(\lambda\xi_i)\ind{\xi_i\le M}]
    & \le \EE_{i-1}[(1+\lambda\xi_i+\phi(\lambda\xi_i))\ind{\xi_i\le M}] \notag      \\
    & \le 1 + \underbrace{\EE_{i-1}[\phi(\lambda\xi_i)\ind{\xi_i\le0}]}_{\text{(A)}}
    + \underbrace{\EE_{i-1}[\phi(\lambda\xi_i)\ind{0<\xi_i\le M}]}_{\text{(B)}}
    \label{eq:split}
  \end{align}
  where in the last step we split $\xi_i$ at $0$ as well as at $M$. We bound
  each of (A) and (B).

  \noindent\textbf{(A):} Since $\phi(u) \le u^2/2$ for $u\le0$,
  together with \eqref{eq:var},
  \begin{equation}
    \EE_{i-1}[\phi(\lambda\xi_i)\ind{\xi_i\le0}] \le \frac{\lambda^2}{2}\EE_{i-1}[\xi_i^2]
    \le 2\lambda^2 J_\beta K_{i-1}^2.
    \label{eq:neg}
  \end{equation}

  \noindent\textbf{(B):} Since $\phi(\lambda z) = \int_0^z
  \lambda(e^{\lambda t}-1)\, dt$,
  $\phi(\lambda\xi_i)\ind{0<\xi_i\le M} = \int_0^M \lambda(e^{\lambda
  t}-1)\ind{t<\xi_i\le M}\, dt$, and $\ind{t<\xi_i\le M} \le \ind{|\xi_i|>t}$.
  Using Fubini's theorem, \eqref{eq:tail}, and $e^v-1\le ve^v$ ($v\ge0$) in
  turn, and substituting $t=K_{i-1} u$,
  \[
    \EE_{i-1}[\phi(\lambda\xi_i)\ind{0<\xi_i\le M}]
    \le 2\lambda^2 K_{i-1}^2 \int_0^{M/K_{i-1}} u\,e^{\lambda K_{i-1} u}\,\overline{\Psi}_\beta(u)\, du.
  \]
  Since $\lambda$ satisfies \eqref{eq:lambbound}, $\lambda K_{i-1} \le
  \log^\beta(1+M/K_{i-1})/(2M/K_{i-1})$, so by Lemma~\ref{lem:key}, for $u\in[0,M/K_{i-1}]$,
  $\lambda K_{i-1} u \le (\log^\beta(1+u)+\beta^\beta)/2$, and hence
  \begin{align*}
    \int_0^{M/K_{i-1}} u\,e^{\lambda K_{i-1} u}\,\overline{\Psi}_\beta(u)\, du
    & \le \exp(\beta^\beta/2) \int_0^\infty u\,\exp\!\left(-\tfrac12\log^\beta(1+u)\right) du \\
    & = \exp(\beta^\beta/2) \int_0^\infty (e^s-1)e^s\exp(-s^\beta/2)\, ds
    \le \exp(\beta^\beta/2) J_\beta,
  \end{align*}
  i.e.\ $\EE_{i-1}[\phi(\lambda\xi_i)\ind{0<\xi_i\le M}] \le 2\lambda^2
  \exp(\beta^\beta/2) J_\beta K_{i-1}^2$. Substituting these bounds, together with\\
  $1+\exp(\beta^\beta/2) \le 2\exp(\beta^\beta/2)$ and $1+v\le e^v$, into
  \eqref{eq:split},
  \[
    \EE_{i-1}[e^{\lambda\xi_i}\ind{\xi_i\le M}]
    \le 1 + 4\exp(\beta^\beta/2) J_\beta \lambda^2 K_{i-1}^2
    \le \exp\!\left(\frac{\lambda^2 a_\beta K_{i-1}^2}{2}\right).
  \]
\end{proof}

Having established the necessary lemmas, we now prove
Proposition~\ref{thm:freedman}.

\begin{proof}[Proof of Proposition~\ref{thm:freedman}]
  First, noting \eqref{eq:tail}, $K_{i-1}\le\mbar$, monotonicity, and the
  definition $\log^\beta(1+M/\mbar) = \Lambda$,
  \begin{equation}
    \forall i\in[N], \quad \PP(\xi_i>M) \le 2e^{-\Lambda} = \frac{2\delta}{N}
    \implies \PP\!\left(\max_{i\in[N]}\xi_i>M\right) \le 2\delta.
    \label{eq:trunc}
  \end{equation}

  Next fix $\tilde\lambda\in[0,1/(\bF\mbar)]$ and define
  \[
    Z_0 := 1, \qquad
    Z_k := \exp\!\left(\tilde\lambda S_k - \frac{\tilde\lambda^2 V_k}{2}\right)
    \prod_{i\in[k]} \ind{\xi_i\le M} \quad \textrm{for each}\quad k\in[N].
  \]
  We show $Z_k$ is a nonnegative supermartingale. First, from
  $K_{k-1}\le\mbar$ and
  \[\frac{1}{\bF\mbar} = \frac{\Lambda}{2M} =
    \frac{1}{2M}\log^\beta(1+M/\mbar),
  \]
  \[
    \tilde\lambda \le \frac1{\bF\mbar} = \frac1{2M}\log^\beta\!\left(1+\frac M\mbar\right)
    \le \frac1{2M}\log^\beta\!\left(1+\frac{M}{K_{k-1}}\right),
  \]
  where the last inequality also covers $K_{k-1}=0$, its right-hand side
  being infinite under the convention stated before
  Proposition~\ref{thm:freedman}.

  Next consider the conditional expectation of $Z_k$ given $\cF_{k-1}$.
  Applying Lemma~\ref{lem:mgf} with $i=k$ and $\lambda=\tilde\lambda$,
  \[
    \EE[Z_k\mid\cF_{k-1}] = Z_{k-1}\exp\!\left(-\frac{\tilde\lambda^2 a_\beta K_{k-1}^2}{2}\right)
    \EE[e^{\tilde\lambda\xi_k}\ind{\xi_k\le M}\mid\cF_{k-1}] \le Z_{k-1},
  \]
  so $\{Z_k\}$ is a nonnegative supermartingale with $Z_0=1$. For
  \eqref{eq:main2} we take $\tilde\lambda = \lambda$ with
  $\lambda\in[0,1/(\bF\mbar)]$ as in the statement. Define the event
  $T_k := \{\max_{i\le k} \xi_i \le M\}$. Setting $A_k := \{S_k \ge x
  \text{ and } V_k \le \gamma\}$,
  \[
    A_k \cap T_k \implies Z_k = \exp\!\left(\lambda S_k - \frac{\lambda^2}{2}V_k\right)
    \ge \exp\!\left(\lambda x - \frac{\lambda^2}{2}\gamma\right).
  \]
  Hence
  \begin{equation}
    \bigcup_{k\in[N]} A_k \subseteq
    \left\{\max_{k\in[N]} Z_k \ge \exp\!\left(\lambda x - \frac{\lambda^2}{2}\gamma\right)\right\}
    \cup \left\{\max_{i\in[N]} \xi_i > M\right\}
    \label{eq:union_bound}
  \end{equation}
  and, by Lemma~\ref{lem:ville} and \eqref{eq:trunc},
  \[
    \PP\!\left(\bigcup_{k\in[N]} A_k\right)
    \le \exp\!\left(-\lambda x + \frac{\lambda^2}{2}\gamma\right) + 2\delta,
  \]
  which establishes \eqref{eq:main2}.

  Finally we show \eqref{eq:main1}. Let $\alpha\ge \bF\mbar$,
  $\lambda\in[0,1/(2\alpha)]$, and set $\tilde\lambda := 2\lambda$, so that
  $\tilde\lambda \le 1/\alpha \le 1/(\bF\mbar)$, so the argument above applies.
  Set $A_k := \{S_k \ge x \text{ and } V_k \le \alpha S_k + \gamma\}$. On
  $A_k\cap T_k$, we have
  \[
    \tilde\lambda S_k - \frac{\tilde\lambda^2}{2}V_k
    \ge 2\lambda(1-\lambda\alpha)S_k - 2\lambda^2\gamma
    \ge \lambda x - 2\lambda^2\gamma,
  \]
  using $\lambda\alpha\le1/2$, hence $2\lambda(1-\lambda\alpha)\ge\lambda\ge0$,
  together with $S_k\ge x\ge0$. Hence on $A_k\cap T_k$,
  $Z_k \ge \exp(\lambda x - 2\lambda^2\gamma)$, and arguing as in
  \eqref{eq:union_bound},
  \[
    \PP\!\left(\bigcup_{k\in[N]} A_k\right)
    \le \exp(-\lambda x + 2\lambda^2\gamma) + 2\delta,
  \]
  which establishes \eqref{eq:main1}.
\end{proof}

\section{Proofs of the Main Results}
\label{sec:proofs-main}

\begin{proof}[Proof of Theorem~\ref{thm:main1}]
  Throughout the proof we write $\epsilon_t := \nabla \ell(\bm{w}_t; z_{j_t}) -
  \nabla \cL_S(\bm{w}_t)$ for the gradient noise at step $t$, as in
  Assumption~\ref{asm:gradnoise}. Since $\cL_S$ is an average of the functions $\ell(\cdot\,;z_i)$, each of
  which is $b$-smooth by Assumption~\ref{asm:smooth}, $\cL_S$ is itself
  $b$-smooth, so Lemma~\ref{thm:smooth} applies to $\cL_S$.

  We use Lemma~\ref{thm:smooth} to bound $\cL_S(\bm{w}_{t+1}) - \cL_S(\bm{w}_t)$:
  \begin{align}
    \cL_S(\bm{w}_{t+1}) - \cL_S(\bm{w}_t)
    & \le \inner{\bm{w}_{t+1}-\bm{w}_t}{\nabla \cL_S(\bm{w}_t)}
    + \tfrac12 b \|\bm{w}_{t+1}-\bm{w}_t\|^2 \notag                   \\
    & = -\eta_t \inner{\epsilon_t}{\nabla \cL_S(\bm{w}_t)}
    - \eta_t \|\nabla \cL_S(\bm{w}_t)\|^2
    + \tfrac12 b \eta_t^2 \|\nabla \ell(\bm{w}_t; z_{j_t})\|^2 \notag \\
    & \le -\eta_t \inner{\epsilon_t}{\nabla \cL_S(\bm{w}_t)}
    - (\eta_t - b \eta_t^2) \|\nabla \cL_S(\bm{w}_t)\|^2
    + b \eta_t^2 \|\epsilon_t\|^2 \notag                              \\
    & \le -\eta_t \inner{\epsilon_t}{\nabla \cL_S(\bm{w}_t)}
    - \tfrac12 \eta_t \|\nabla \cL_S(\bm{w}_t)\|^2
    + b \eta_t^2 \|\epsilon_t\|^2,
    \label{eq:tele}
  \end{align}
  where the third line uses $\|\nabla \ell(\bm{w}_t;z_{j_t})\|^2 =
  \|\epsilon_t + \nabla \cL_S(\bm{w}_t)\|^2 \le 2\|\epsilon_t\|^2 +
  2\|\nabla \cL_S(\bm{w}_t)\|^2$, and the last line uses $b\eta_t^2 \le
  \tfrac12\eta_t$, which follows from the assumption $\eta_t \le 1/(2b)$.

  Summing \eqref{eq:tele} from $t=1$ to $T$,
  \begin{align*}
    \cL_S(\bm{w}_{T+1}) - \cL_S(\bm{w}_1)
    \le & -\sum_{t\in[T]} \eta_t \inner{\epsilon_t}{\nabla \cL_S(\bm{w}_t)}
    - \tfrac12 \sum_{t\in[T]} \eta_t \|\nabla \cL_S(\bm{w}_t)\|^2   + \sum_{t\in[T]} b \eta_t^2 \|\epsilon_t\|^2,
  \end{align*}
  and since $\bm{w}(S)$ is a minimizer of $\cL_S$ over $\cW$ we have
  $\cL_S(\bm{w}_{T+1}) \ge \cL_S(\bm{w}(S))$, whence we obtain
  \begin{align}
    \sum_{t\in[T]} \eta_t \|\nabla \cL_S(\bm{w}_t)\|^2
    \le \  & 2\left(\cL_S(\bm{w}_1) - \cL_S(\bm{w}(S))\right) \notag                                                                                                                \\
    & \underbrace{-\sum_{t\in[T]} 2\eta_t \inner{\epsilon_t}{\nabla \cL_S(\bm{w}_t)}}_{\text{(A)}}  + \underbrace{\sum_{t\in[T]} 2b \eta_t^2 \|\epsilon_t\|^2}_{\text{(B)}}.
    \label{eq:inter1}
  \end{align}
  We now bound the terms (A) and (B) in \eqref{eq:inter1}.

  \textbf{(A)} With $K$ denoting the uniform bound on the conditional Orlicz norms in
  Assumption~\ref{asm:gradnoise}, set
  \[
    \xi_t := -2\eta_t \inner{\epsilon_t}{\nabla \cL_S(\bm{w}_t)},
    \qquad
    K_{t-1} := 2\eta_t \|\nabla \cL_S(\bm{w}_t)\| K.
  \]
  We verify the hypotheses of Proposition~\ref{thm:freedman} with $N=T$, so
  that it applies to $S_k = \sum_{t\in[k]}\xi_t$.

  \emph{Adaptedness and (A1).} The iterate $\bm{w}_t$ is
  $\cF_{t-1}$-measurable, hence so is $K_{t-1}$, and
  Assumption~\ref{asm:cap} gives the deterministic bound
  \[
    0 \le K_{t-1} = 2\eta_t\|\nabla \cL_S(\bm{w}_t)\| K \le 2GK =: \mbar .
  \]
  \emph{(A2).} $\EE[\xi_t\mid\cF_{t-1}] = -2\eta_t \inner{\EE[\epsilon_t\mid
  \cF_{t-1}]}{\nabla \cL_S(\bm{w}_t)} = 0$ by the mean-zero condition in Assumption~\ref{asm:gradnoise}, since
  $\eta_t\nabla \cL_S(\bm{w}_t)$ is $\cF_{t-1}$-measurable.

  \emph{(A3).} By the Cauchy--Schwarz inequality,
  \[
    |\xi_t| = 2\eta_t \left|\inner{\epsilon_t}{\nabla \cL_S(\bm{w}_t)}\right|
    \le 2\eta_t \|\nabla \cL_S(\bm{w}_t)\| \cdot \|\epsilon_t\| .
  \]
  On the $\cF_{t-1}$-measurable event $\{K_{t-1}>0\}$, this gives
  $|\xi_t|/K_{t-1}\le\|\epsilon_t\|/K$. Thus monotonicity of $\Psi_\beta$
  and Assumption~\ref{asm:gradnoise} yield, on this event,
  \[
    \EE\!\left[\Psi_\beta\!\left(\frac{|\xi_t|}{K_{t-1}}\right)
    \;\middle|\;\cF_{t-1}\right]
    \le \EE\!\left[\Psi_\beta\!\left(\frac{\|\epsilon_t\|}{K}\right)
    \;\middle|\;\cF_{t-1}\right] \le 1,
  \]
  which is precisely (A3). In the
  degenerate case $\nabla \cL_S(\bm{w}_t)=0$, for which $K_{t-1}=0$, the
  displayed Cauchy--Schwarz bound already gives $\xi_t=0$ almost surely, so
  (A3) again holds under the convention stated before
  Proposition~\ref{thm:freedman}.

  We may therefore apply Proposition~\ref{thm:freedman}(i) with $\gamma=0$,
  $\alpha \ge \bF\mbar$, $\lambda = 1/(2\alpha)$ and $x = 2\alpha\log(1/\delta)$,
  which yields
  \[
    \PP\!\left(\bigcup_{k\in[T]}\left\{S_k \ge x \text{ and } V_k \le \alpha S_k\right\}\right)
    \le \exp(-\lambda x) + 2\delta = \delta + 2\delta = 3\delta .
  \]
  On the complement, every $k\in[T]$ satisfies $S_k < x$ or $S_k < V_k/\alpha$,
  hence $S_k \le \max(x, V_k/\alpha) \le x + V_k/\alpha$, where by definition
  \[
    V_T = \sum_{t\in[T]} a_\beta K_{t-1}^2
    = 4a_\beta K^2 \sum_{t\in[T]} \eta_t^2 \|\nabla \cL_S(\bm{w}_t)\|^2.
  \]
  We conclude that with probability at least $1-3\delta$,
  \begin{equation}
    -\sum_{t\in[T]} 2\eta_t \inner{\epsilon_t}{\nabla \cL_S(\bm{w}_t)}
    \le 2\alpha \log\left(1/\delta\right)
    + \frac{4a_\beta K^2}{\alpha} \sum_{t\in[T]} \eta_t^2 \|\nabla \cL_S(\bm{w}_t)\|^2.
    \label{bound:A}
  \end{equation}

  \textbf{(B)} Assumption~\ref{asm:gradnoise}, Proposition~\ref{thm:square} and
  the homogeneity of the Orlicz norm (Lemma~\ref{lem:orlicz-props}(i)) give
  \[
    \left\|\,2b\eta_t^2 \|\epsilon_t\|^2\right\|_{\Psi_{\beta,2^{-\beta}}}
    \le 2b \eta_t^2 \left\|\,\|\epsilon_t\|\,\right\|_{\Psi_\beta}^2
    \le 2b \eta_t^2 K^2.
  \]
  Hence, by Corollary~\ref{thm:sum-beta-heavy} in the form
  \eqref{eq:sum-hp} with $c = 2^{-\beta}$, with probability at least $1-\delta$,
  \begin{equation}
    \sum_{t\in[T]} 2b \eta_t^2 \|\epsilon_t\|^2
    \le 2g_\beta bK^2 \exp\!\left(2\log^{1/\beta}(2/\delta)\right) \sum_{t\in[T]} \eta_t^2.
    \label{bound:B}
  \end{equation}

  Substituting \eqref{bound:A} and \eqref{bound:B} into \eqref{eq:inter1}, a
  union bound gives that with probability at least $1-4\delta$,
  \begin{align*}
    \sum_{t\in[T]} \eta_t \|\nabla \cL_S(\bm{w}_t)\|^2
    \le \  & 2\left(\cL_S(\bm{w}_1) - \cL_S(\bm{w}(S))\right)                                                            \\
    & + 2\alpha \log(1/\delta) + \frac{4a_\beta K^2}{\alpha} \sum_{t\in[T]} \eta_t^2 \|\nabla \cL_S(\bm{w}_t)\|^2 \\
    & + 2g_\beta bK^2 \exp\!\left(2\log^{1/\beta}(2/\delta)\right)\sum_{t\in[T]} \eta_t^2.
  \end{align*}
  Rearranging in terms of $\sum_{t\in[T]} \eta_t \|\nabla \cL_S(\bm{w}_t)\|^2$,
  \begin{align*}
    \sum_{t\in[T]}\left(1-\frac{4a_\beta K^2\eta_t}{\alpha}\right)\eta_t \|\nabla \cL_S(\bm{w}_t)\|^2
    \le \  & 2\left(\cL_S(\bm{w}_1) - \cL_S(\bm{w}(S))\right)                                                              \\
    & + 2\alpha \log(1/\delta) + 2g_\beta bK^2 \exp\!\left(2\log^{1/\beta}(2/\delta)\right)\sum_{t\in[T]} \eta_t^2.
  \end{align*}
  Since $\eta_t$ is nonincreasing in $t$, imposing $\alpha \ge 8a_\beta
  K^2\eta_1$ makes $4a_\beta K^2\eta_t/\alpha \le 1/2$ for every $t$, so
  the left-hand side is at least $\tfrac12\sum_{t\in[T]}\eta_t\|\nabla
  \cL_S(\bm{w}_t)\|^2$. Multiplying through by $2$,
  \begin{align}
    \sum_{t\in[T]}\eta_t \|\nabla \cL_S(\bm{w}_t)\|^2
    \le \  & 4\left(\cL_S(\bm{w}_1) - \cL_S(\bm{w}(S))\right) \notag                                                       \\
    & + 4\alpha \log(1/\delta) + 4g_\beta bK^2 \exp\!\left(2\log^{1/\beta}(2/\delta)\right)\sum_{t\in[T]} \eta_t^2.
    \label{eq:bound1}
  \end{align}

  It remains to choose $\alpha$ and to convert the confidence level. Taking
  $\alpha := \max(8a_\beta K^2\eta_1, \bF\mbar)$, both constraints imposed
  above are met, and \eqref{eq:bound1} holds with probability at least
  $1-4\delta$. Recalling \eqref{eq:freedef} with $N=T$, so that $\Lambda =
  \log(T/\delta)$, and $\mbar = 2GK$ from
  Assumption~\ref{asm:cap},
  \[
    \bF\mbar = \frac{2\left(\exp\!\left(\log^{1/\beta}(T/\delta)\right)-1\right)}{\log(T/\delta)}\,2GK
    = \mathcal{O}\!\left(\exp\!\left(\log^{1/\beta}(T/\delta)\right)\right),
  \]
  so that, bounding $\log(1/\delta) \le \log(e/\delta)$,
  $\alpha\log(1/\delta) = \mathcal{O}(\exp(\log^{1/\beta}(T/\delta))
  \log(e/\delta))$. Both this term and the coefficient
  $g_\beta\exp(2\log^{1/\beta}(2/\delta))$ of $\sum_{t\in[T]}\eta_t^2$ are dominated by
  $\exp(2\log^{1/\beta}(T/\delta))$ up to a factor depending only on $\beta$ for
  $T\ge2$. Since $\bm{w}_1=\bm{0}$, the initial-gap assumption in
  Section~\ref{sec:main-results} bounds the
  initial-gap term in \eqref{eq:bound1} by $4\Delta_0$ almost surely.
  This deterministic bound is absorbed into $\mathcal{O}(\cdot)$ because
  $\log(e/\delta)\ge1$; the implicit constant may depend on $\Delta_0$,
  but not on $S,n,d,T,\delta$. Finally,
  replacing $\delta$ by $\delta/4$, we obtain that with probability
  at least $1-\delta$,
  \[
    \sum_{t\in[T]} \eta_t \|\nabla \cL_S(\bm{w}_t)\|^2
    = \mathcal{O}\!\left(\exp\!\left(2\log^{1/\beta}(T/\delta)\right)
    \left(\log\tfrac e\delta+\sum_{t\in[T]}\eta_t^2\right)\right).
  \]
\end{proof}

\begin{proof}[Proof of Theorem~\ref{thm:main2}]
  Rewriting the SGD update rule as\\
  $\bm{w}_{t+1} = \bm{w}_t - \eta_t(\nabla \ell(\bm{w}_t;  z_{j_t}) - \nabla \cL_S(\bm{w}_t) + \nabla \cL_S(\bm{w}_t))$, we get
  \begin{align*}
    \bm{w}_{t+1}                           & = \sum_{i\in[t]} \left(-\eta_i(\nabla \ell(\bm{w}_i;z_{j_i}) - \nabla \cL_S(\bm{w}_i))\right)
    - \sum_{i\in[t]} \eta_i \nabla \cL_S(\bm{w}_i)                                                                                         \\
    \Longrightarrow \quad \|\bm{w}_{t+1}\| & \le
    \left\|\sum_{i\in[t]} -\eta_i(\nabla \ell(\bm{w}_i;z_{j_i}) - \nabla \cL_S(\bm{w}_i))\right\|
    + \left\|\sum_{i\in[t]} \eta_i \nabla \cL_S(\bm{w}_i)\right\|.
  \end{align*}
  For the first term, the vectors $\eta_i(\nabla \ell(\bm{w}_i;z_{j_i}) -
  \nabla \cL_S(\bm{w}_i))$ form an $\mathbb{R}^d$-valued MDS by
  the mean-zero condition in Assumption~\ref{asm:gradnoise}. Taking expectations
  in the Orlicz bound in Assumption~\ref{asm:gradnoise} and
  using the homogeneity of the Orlicz norm (Lemma~\ref{lem:orlicz-props}(i))
  bounds their Orlicz norms by $\sigma_i = \eta_i K$. Corollary~\ref{cor:delta}
  therefore gives, for any $\delta\in(0,1)$, with probability $1-\delta$,
  \[
    \max_{t\in[T]} \left\|\sum_{i\in[t]} \eta_i(\nabla \ell(\bm{w}_i;z_{j_i}) - \nabla \cL_S(\bm{w}_i))\right\|
    = \mathcal{O}\!\left(\sqrt{\log\tfrac2\delta \sum_{i\in[T]} \eta_i^2}
    \,\exp\!\left(\log^{1/\beta}(4/\delta)\right)\right),
  \]
  where $\Psi_\beta^{-1}(2/\delta) = \exp(\log^{1/\beta}(1+2/\delta))-1 \le
  \exp(\log^{1/\beta}(4/\delta))$ because $\delta\le2$.

  For the second term, by the Cauchy--Schwarz inequality and
  \eqref{eq:bound1}, for every $t\in[T]$,
  \begin{align*}
    \left\|\sum_{i\in[t]} \eta_i \nabla \cL_S(\bm{w}_i)\right\|^2
    & \le \left(\sum_{i\in[t]} \eta_i\right)\left(\sum_{i\in[t]} \eta_i \|\nabla \cL_S(\bm{w}_i)\|^2\right) \\
    & \le \left(\sum_{i\in[T]} \eta_i\right)
    \mathcal{O}\!\left(\exp\!\left(2\log^{1/\beta}(T/\delta)\right)
    \left(\log\tfrac e\delta+\sum_{t\in[T]}\eta_t^2\right)\right),
  \end{align*}
  where both factors were enlarged from $t$ to $T$ using the nonnegativity of
  the summands. Note that the second term is bounded by the square root of
  this quantity. Abbreviating
  \begin{align}
    \Sigma_1 := \sum_{t\in[T]}\eta_t, \qquad
    \Sigma_2 := \sum_{t\in[T]}\eta_t^2, \qquad
    E_\beta := \exp\!\left(2\log^{1/\beta}(T/\delta)\right), \qquad
    A := \log\tfrac e\delta+\Sigma_2,
    \label{eq:etaabbr}
  \end{align}
  and combining the two displays, we obtain
  \begin{equation}
    \max_{t\in[T]}\|\bm{w}_{t+1}\|
    \le \mathcal{O}\!\left(\sqrt{E_\beta \Sigma_2\log\tfrac e\delta}\right)
    + \mathcal{O}\!\left(\sqrt{\Sigma_1 E_\beta A}\right),
    \label{eq:worder}
  \end{equation}
  where the first term uses $\exp(\log^{1/\beta}(4/\delta)) =
  \mathcal{O}(\exp(\log^{1/\beta}(T/\delta))) = \mathcal{O}(\sqrt{E_\beta})$,
  by Lemma~\ref{lm:rescale} with $a=4$ and $1/\delta \le T/\delta$, and
  $\log(2/\delta) \le \log(e/\delta)$. Squaring \eqref{eq:worder} and using
  $\Sigma_2 \le \eta_1\Sigma_1$, which holds because $\eta_t$ is nonincreasing,
  together with $\log(e/\delta) \le A$,
  \begin{equation}
    \max_{t\in[T+1]}\|\bm{w}_t\|^2
    = \mathcal{O}\!\left(E_\beta \Sigma_2\log\tfrac e\delta + \Sigma_1 E_\beta A\right)
    = \mathcal{O}\!\left((1+\Sigma_1) E_\beta A\right)
    \label{eq:worder2}
  \end{equation}
  (the case $t=1$ being trivial, as $\bm{w}_1=\bm{0}$).

  By the initial-gap assumption in Section~\ref{sec:main-results}, the implicit constant in \eqref{eq:worder2}
  can be chosen deterministically, independently of $S,n,d,T,\delta$.
  Choose such a constant $C_R>0$ and set
  $R:=\sqrt{C_R(1+\Sigma_1)E_\beta A}$.
  Thus $R$ is deterministic, and every iterate $\bm{w}_1,\dots,\bm{w}_{T+1}$ lies in the ball $B_R$ on
  the event above. Since Lemma~\ref{lm:gradgen} bounds the generalization
  error of the gradient uniformly over $B_R$, it applies simultaneously to all
  iterates, giving
  \[
    \max_{t\in[T+1]}\|\nabla \cL(\bm{w}_t) - \nabla \cL_S(\bm{w}_t)\|
    \le \frac{b R + b'}{\sqrt n}
    \left(2 + 2\sqrt{48 e \sqrt2 (\log 2 + d\log(3e))} + \sqrt{2\log(1/\delta)}\right).
  \]
  Squaring both sides, bounding $(bR+b')^2 = \mathcal{O}(1+R^2)$ and
  substituting \eqref{eq:worder2}, we obtain
  \[
    \max_{t\in[T+1]}\|\nabla \cL(\bm{w}_t) - \nabla \cL_S(\bm{w}_t)\|^2
    \le \mathcal{O}\!\left(\frac{(1+\Sigma_1) E_\beta A \left(d+\log(1/\delta)\right)}{n}\right),
  \]
  which is \eqref{eq:gen-unif} after a union bound over the two events used
  above and a rescaling of $\delta$ by a constant factor, both absorbed into
  the $\mathcal{O}(\cdot)$ by Lemma~\ref{lm:rescale}.
\end{proof}

Finally, we prove Theorem~\ref{thm:main3}.

\begin{proof}[Proof of Theorem~\ref{thm:main3}]
  We retain the abbreviations $\Sigma_1,\Sigma_2,E_\beta,A$ introduced in the
  proof of Theorem~\ref{thm:main2}. Splitting the population gradient into the
  empirical gradient and the generalization error and using
  $\|\bm{u}+\bm{v}\|^2 \le 2\|\bm{u}\|^2+2\|\bm{v}\|^2$, we obtain
  \begin{align}
    \sum_{t\in[T]} \eta_t \|\nabla \cL(\bm{w}_t)\|^2
    & \le 2\sum_{t\in[T]} \eta_t \|\nabla \cL(\bm{w}_t) - \nabla \cL_S(\bm{w}_t)\|^2
    + 2\sum_{t\in[T]} \eta_t \|\nabla \cL_S(\bm{w}_t)\|^2 \notag                        \\
    & \le 2\Sigma_1 \max_{t\in[T]} \|\nabla \cL(\bm{w}_t) - \nabla \cL_S(\bm{w}_t)\|^2
    + 2\sum_{t\in[T]} \eta_t \|\nabla \cL_S(\bm{w}_t)\|^2.
    \label{eq:grad-sum-bound}
  \end{align}
  Because the first term involves the maximum over the whole trajectory,
  Theorem~\ref{thm:main2} was stated in the uniform form
  \eqref{eq:gen-unif}.

  Applying Theorem~\ref{thm:main2} to the first term of
  \eqref{eq:grad-sum-bound} and Theorem~\ref{thm:main1} to the second, and
  taking a union bound over the two events, we obtain with probability at
  least $1-2\delta$,
  \begin{align}
    \sum_{t\in[T]} \eta_t \|\nabla \cL(\bm{w}_t)\|^2
    & = \mathcal{O}\!\left(\frac{\Sigma_1(1+\Sigma_1) E_\beta A\left(d+\log(1/\delta)\right)}{n}\right)
    + \mathcal{O}\!\left(E_\beta A\right) \notag                                                         \\
    & = \mathcal{O}\!\left(E_\beta A
    \left(1 + \frac{\Sigma_1(1+\Sigma_1)\left(d+\log(1/\delta)\right)}{n}\right)\right).
    \label{eq:grad-bound}
  \end{align}
  Dividing \eqref{eq:grad-bound} by $\Sigma_1$ gives
  \[
    \frac{\sum_{t\in[T]} \eta_t \|\nabla \cL(\bm{w}_t)\|^2}{\Sigma_1}
    = \mathcal{O}\!\left(E_\beta A
    \left(\frac{1}{\Sigma_1} + \frac{(1+\Sigma_1)\left(d+\log(1/\delta)\right)}{n}\right)\right),
  \]
  which is \eqref{eq:main3} once $E_\beta$ and $A$ are expanded and $\delta$
  is replaced by $\delta/2$, the resulting change of $\log(e/\delta)$ into
  $\log(2e/\delta)$ and of $T/\delta$ into $2T/\delta$ being absorbed by the
  $\mathcal{O}(\cdot)$ by Lemma~\ref{lm:rescale}.
\end{proof}

\begin{proof}[Proof of Corollary~\ref{cor:horizon}]
  We write $D := d+\log(1/\delta)$, and let $c_2 \ge c_1 > 0$ be constants for
  which \eqref{eq:budget} reads $c_1\sqrt{n/D} \le \Sigma_1 \le c_2\sqrt{n/D}$.
  Since $n \ge D$, the lower bound gives $\Sigma_1 \ge c_1$, so
  $1+\Sigma_1 \le (1+1/c_1)\Sigma_1$. Substituting the two bounds into the
  last factor of \eqref{eq:main3}, we obtain
  \[
    \frac{1}{\Sigma_1} + \frac{(1+\Sigma_1)D}{n}
    \le \frac{1}{c_1}\sqrt{\frac{D}{n}}
    + \left(1+\frac{1}{c_1}\right)c_2\sqrt{\frac{n}{D}}\cdot\frac{D}{n}
    = \left(\frac{1}{c_1}+\left(1+\frac{1}{c_1}\right)c_2\right)\sqrt{\frac{D}{n}},
  \]
  and the claim follows from Theorem~\ref{thm:main3}.
\end{proof}

\section{Step-Size Instantiations}
\label{sec:proofs-instantiation}

\begin{proposition}[$1/t$ decay]
  \label{lem:decay-1-t}
  Fix $\eta_0 \in (0,1/(2b)]$ and set $\eta_t := \frac{\eta_0}{t}$ for
  $t=1,\dots,T$. Then the step-size sums satisfy
  \begin{equation}
    \sum_{t\in[T]} \eta_t = \eta_0(\log T + \mathcal{O}(1)), \qquad
    \sum_{t\in[T]} \eta_t^2 = \mathcal{O}(\eta_0^2).
    \label{eq:decay-1t-sum}
  \end{equation}
\end{proposition}

\begin{proof}[Proof of Proposition~\ref{lem:decay-1-t}]
  Integral comparison gives
  $\log(T+1)\le\sum_{t\in[T]}t^{-1}\le1+\log T$ and
  $\sum_{t\in[T]}t^{-2}\le2$. Multiplying by $\eta_0$ and $\eta_0^2$,
  respectively, proves \eqref{eq:decay-1t-sum}.
\end{proof}

\begin{proposition}[$1/\sqrt t$ decay]
  \label{lem:decay-1-sqrt-t}
  Fix $\eta_0 \in (0,1/(2b)]$ and set $\eta_t := \frac{\eta_0}{\sqrt t}$ for
  $t=1,\dots,T$. Then the step-size sums satisfy
  \begin{equation}
    \sum_{t\in[T]} \eta_t = \eta_0(2\sqrt T + \mathcal{O}(1)), \qquad
    \sum_{t\in[T]} \eta_t^2 = \eta_0^2(\log T + \mathcal{O}(1)).
    \label{eq:decay-1sqrtt-sum}
  \end{equation}
\end{proposition}

\begin{proof}[Proof of Proposition~\ref{lem:decay-1-sqrt-t}]
  Integral comparison gives
  $2\sqrt{T+1}-2\le\sum_{t\in[T]}t^{-1/2}\le2\sqrt T-1$.
  This and the harmonic-sum estimate in the preceding proof yield
  \eqref{eq:decay-1sqrtt-sum} after scaling by $\eta_0$ and $\eta_0^2$,
  respectively.
\end{proof}

\begin{proposition}[Cosine decay]
  \label{lem:decay-cosine}
  Fix $\eta_0\in(0,1/(2b)]$ and the number of training steps $T\in\mathbb{N}$,
  $T\ge2$, and set $\eta_t := \frac{\eta_0}{2}\left(1+\cos\!\left(\frac{\pi
  t}{T}\right)\right)$ for $t=1,\dots,T$. Then the step-size sums satisfy
  \begin{equation}
    \sum_{t\in[T]} \eta_t = \frac{\eta_0}{2}(T-1), \qquad
    \sum_{t\in[T]} \eta_t^2 = \frac{\eta_0^2}{8}(3T-4).
    \label{eq:decay-cosine-sum}
  \end{equation}
\end{proposition}

\begin{proof}[Proof of Proposition~\ref{lem:decay-cosine}]
  For $T\ge2$, the standard trigonometric sums give
  \[
    \sum_{t\in[T]}\cos\!\left(\frac{\pi t}{T}\right)=-1,
    \qquad
    \sum_{t\in[T]}\cos^2\!\left(\frac{\pi t}{T}\right)
    =\frac12\sum_{t\in[T]}\left(1+\cos\!\left(\frac{2\pi t}{T}\right)\right)
    =\frac T2.
  \]
  Expanding $\eta_t$ and $\eta_t^2$ and summing proves
  \eqref{eq:decay-cosine-sum}.
\end{proof}

\section{Proof of the Bound under the PL Condition}
\label{sec:proofs-pl}

\begin{proof}[Proof of Theorem~\ref{thm:pl}]
  We keep the notation $\epsilon_t = \nabla \ell(\bm{w}_t;z_{j_t}) - \nabla
  \cL_S(\bm{w}_t)$ and $\eta_t = \frac{4}{\mu_0(t+t_0)}$ for the step size of
  Theorem~\ref{thm:pl}. From \eqref{eq:tele} we already have
  \[
    \cL_S(\bm{w}_{t+1}) - \cL_S(\bm{w}_t)
    \le -\eta_t \inner{\epsilon_t}{\nabla \cL_S(\bm{w}_t)}
    - \tfrac12 \eta_t \|\nabla \cL_S(\bm{w}_t)\|^2
    + b \eta_t^2 \|\epsilon_t\|^2.
  \]
  The PL condition gives $-\tfrac14 \eta_t\|\nabla \cL_S(\bm{w}_t)\|^2 \le
  -\tfrac12 \mu_0\eta_t \left(\cL_S(\bm{w}_t) - \cL_S(\bm{w}(S))\right)$, where
  $\tfrac12\mu_0\eta_t = 2/(t+t_0)$. Substituting it into the above, we obtain
  \begin{multline}
    \tfrac14 \eta_t \|\nabla \cL_S(\bm{w}_t)\|^2
    + \cL_S(\bm{w}_{t+1}) - \cL_S(\bm{w}(S))\\
    \le \left(1 - \frac{2}{t+t_0}\right)
    \left(\cL_S(\bm{w}_t) - \cL_S(\bm{w}(S))\right)
    - \eta_t \inner{\epsilon_t}{\nabla \cL_S(\bm{w}_t)}
    + b \eta_t^2 \|\epsilon_t\|^2.
    \label{eq:pl-onestep}
  \end{multline}
  Multiplying \eqref{eq:pl-onestep} by $(t+t_0)(t+t_0-1) > 0$ and using
  \[
    (t+t_0)(t+t_0-1)\eta_t = \frac{4(t+t_0-1)}{\mu_0}, \qquad
    (t+t_0)(t+t_0-1)\eta_t^2 = \frac{16(t+t_0-1)}{\mu_0^2(t+t_0)}
    \le \frac{16}{\mu_0^2},
  \]
  we obtain
  \begin{multline*}
    \frac{t+t_0-1}{\mu_0}\|\nabla \cL_S(\bm{w}_t)\|^2
    + (t+t_0)(t+t_0-1)\left(\cL_S(\bm{w}_{t+1}) - \cL_S(\bm{w}(S))\right)
    \\
    - (t+t_0-1)(t+t_0-2)\left(\cL_S(\bm{w}_t) - \cL_S(\bm{w}(S))\right)
    \le -\frac{4(t+t_0-1)}{\mu_0}\inner{\epsilon_t}{\nabla \cL_S(\bm{w}_t)}
    + \frac{16b}{\mu_0^2}\|\epsilon_t\|^2.
  \end{multline*}
  Summing this inequality from $t=1$ to $T$ telescopes the second and third terms into
  $(T+t_0)(T+t_0-1)\left(\cL_S(\bm{w}_{T+1}) - \cL_S(\bm{w}(S))\right)$ and
  $-t_0(t_0-1)\left(\cL_S(\bm{w}_1) - \cL_S(\bm{w}(S))\right)$, whence
  \begin{multline}
    \sum_{t\in[T]} \frac{t+t_0-1}{\mu_0}\|\nabla \cL_S(\bm{w}_t)\|^2
    + (T+t_0)(T+t_0-1)\left(\cL_S(\bm{w}_{T+1}) - \cL_S(\bm{w}(S))\right) \\
    \le t_0(t_0-1)\left(\cL_S(\bm{w}_1) - \cL_S(\bm{w}(S))\right)
    + \underbrace{\left(-\sum_{t\in[T]} \frac{4(t+t_0-1)}{\mu_0}
    \inner{\epsilon_t}{\nabla \cL_S(\bm{w}_t)}\right)}_{\text{(A)}} \\
    + \underbrace{\sum_{t\in[T]} \frac{16b}{\mu_0^2}\|\epsilon_t\|^2}_{\text{(B)}}.
    \label{eq:pl-master}
  \end{multline}
  We now bound the terms (A) and (B) in \eqref{eq:pl-master}.

  \textbf{(A)} With $K$ denoting the uniform bound on the conditional Orlicz norms in
  Assumption~\ref{asm:gradnoise}, set
  \[
    \xi_t := -\frac{4(t+t_0-1)}{\mu_0}\inner{\epsilon_t}{\nabla \cL_S(\bm{w}_t)},
    \qquad
    K_{t-1} := \frac{4(t+t_0-1)}{\mu_0}\|\nabla \cL_S(\bm{w}_t)\|K.
  \]
  Conditions (A2) and (A3) of Proposition~\ref{thm:freedman} are verified as in
  the proof of Theorem~\ref{thm:main1}, so it remains to bound $K_{t-1}$ by a
  deterministic $m_t$, as (A1) requires. Since $\cL_S$ averages the per-sample
  gradients, $\|\nabla \cL_S(\bm{0})\| \le b'$, and Assumption~\ref{asm:smooth}
  then gives, for every $\bm{w}\in\cW$,
  \begin{equation}
    \|\nabla \cL_S(\bm{w})\|
    \le \|\nabla \cL_S(\bm{w}) - \nabla \cL_S(\bm{0})\| + \|\nabla \cL_S(\bm{0})\|
    \le b\|\bm{w}\| + b',
    \label{eq:pl-growth}
  \end{equation}
  so a deterministic radius confining the trajectory produces such an $m_t$. The
  schedule in Theorem~\ref{thm:pl} meets the step-size condition, so
  \eqref{eq:worder2} in the proof of Theorem~\ref{thm:main2} applies to it. With
  $\Sigma_1 = \mathcal{O}(\log T)$, $\Sigma_2 = \mathcal{O}(1)$ and hence
  $A = \mathcal{O}(\log(e/\delta))$ in the notation \eqref{eq:etaabbr}, it gives
  a deterministic
  \begin{equation}
    R = \mathcal{O}\!\left(\exp\!\left(\log^{1/\beta}(T/\delta)\right)
    \sqrt{\log T\,\log(e/\delta)}\right)
    \label{eq:pl-radius}
  \end{equation}
  for which the event $\mathcal{E}_R := \{\max_{t\in[T+1]}\|\bm{w}_t\| \le R\}$
  has probability at least $1-\delta$.

  Since \eqref{eq:pl-growth} bounds $K_{t-1}$ by
  $m_t := 4(t+t_0-1)(bR+b')K/\mu_0$ on $\mathcal{E}_R$ only, we apply
  Proposition~\ref{thm:freedman} to the sequence frozen outside that event,
  $\tilde\xi_t := \chi_t\xi_t$ and $\tilde K_{t-1} := \chi_t K_{t-1}$ with the
  $\cF_{t-1}$-measurable $\chi_t := \ind{\max_{s\in[t]}\|\bm{w}_s\| \le R}$. A
  predictable indicator preserves (A2) and (A3), so (A1)--(A3) hold with $N=T$
  and $\mbar := \max_{t\in[T]} m_t = 4(T+t_0-1)(bR+b')K/\mu_0$. Part (i) with
  $\gamma = 0$, $\alpha \ge \bF\mbar$, $\lambda = 1/(2\alpha)$ and
  $x = 2\alpha\log(1/\delta)$, argued on the complement of the event in
  \eqref{eq:main1} as in the proof of Theorem~\ref{thm:main1}, together with
  $\tilde\xi_t = \xi_t$ for all $t\in[T]$ on $\mathcal{E}_R$, gives after a union
  bound with $\mathcal{E}_R$ that with probability at least $1-4\delta$,
  \begin{equation}
    -\sum_{t\in[T]} \frac{4(t+t_0-1)}{\mu_0}\inner{\epsilon_t}{\nabla \cL_S(\bm{w}_t)}
    \le 2\alpha\log(1/\delta)
    + \frac{16 a_\beta K^2}{\mu_0^2\alpha}
    \sum_{t\in[T]} (t+t_0-1)^2\|\nabla \cL_S(\bm{w}_t)\|^2.
    \label{eq:pl-A}
  \end{equation}
  The sum in \eqref{eq:pl-A} is absorbed by the gradient sum in
  \eqref{eq:pl-master} once $16 a_\beta K^2(t+t_0-1)/(\mu_0\alpha) \le 1$ for
  every $t\in[T]$, that is, once $\alpha \ge 16 a_\beta K^2(T+t_0-1)/\mu_0$.

  \textbf{(B)} Assumption~\ref{asm:gradnoise}, Proposition~\ref{thm:square} and
  the homogeneity of the Orlicz norm (Lemma~\ref{lem:orlicz-props}(i)) give
  $\left\|\,\frac{16b}{\mu_0^2}\|\epsilon_t\|^2\right\|_{\Psi_{\beta,2^{-\beta}}}
  \le \frac{16bK^2}{\mu_0^2}$. Hence, by Corollary~\ref{thm:sum-beta-heavy} in
  the form \eqref{eq:sum-hp} with $c = 2^{-\beta}$, with probability at least
  $1-\delta$,
  \begin{equation}
    \sum_{t\in[T]} \frac{16b}{\mu_0^2}\|\epsilon_t\|^2
    \le \frac{16g_\beta bK^2}{\mu_0^2} \exp\!\left(2\log^{1/\beta}(2/\delta)\right) T.
    \label{eq:pl-B}
  \end{equation}

  Substituting \eqref{eq:pl-A} and \eqref{eq:pl-B} into \eqref{eq:pl-master},
  discarding the remaining nonnegative gradient sum and taking a union bound, we
  obtain that with probability at least $1-5\delta$,
  \begin{multline}
    (T+t_0)(T+t_0-1)\left(\cL_S(\bm{w}_{T+1}) - \cL_S(\bm{w}(S))\right)
    \le t_0(t_0-1)\left(\cL_S(\bm{w}_1) - \cL_S(\bm{w}(S))\right) \\
    + 2\alpha\log(1/\delta)
    + \frac{16g_\beta bK^2}{\mu_0^2} \exp\!\left(2\log^{1/\beta}(2/\delta)\right) T,
    \label{eq:pl-preassembly}
  \end{multline}
  for every $\alpha \ge \max\left\{16 a_\beta K^2(T+t_0-1)/\mu_0,\
  \bF\mbar\right\}$.

  It remains to choose $\alpha$ and to bound the three terms on the right-hand
  side of \eqref{eq:pl-preassembly}. We take $\alpha$ equal to this maximum,
  which is deterministic because $R$ is, and write
  $E_\beta = \exp(2\log^{1/\beta}(T/\delta))$ as in \eqref{eq:etaabbr}. Recalling
  \eqref{eq:freedef} with $N=T$, we have $\Lambda = \log(T/\delta)$. Since
  $T\ge2$, $t_0\ge1$ and $\delta<1$, we use below that $T+t_0-1\le t_0T$,
  $\log T\ge\log2$ and $\log(1/\delta)\le\Lambda$.

  For the first term, the PL condition \eqref{eq:pl} at $\bm{w}_1=\bm{0}$ and
  Assumption~\ref{asm:pl} give
  $\cL_S(\bm{w}_1)-\cL_S(\bm{w}(S)) \le \|\nabla \cL_S(\bm{0})\|^2/(2\mu_0)
  \le (b')^2/(2\mu_0)$, so this term is at most $t_0^2(b')^2/(2\mu_0)$.

  For the second term, \eqref{eq:freedef} gives
  $\bF \le 2\exp(\Lambda^{1/\beta})/\Lambda$. By \eqref{eq:pl-radius} and
  $\exp(\Lambda^{1/\beta})\sqrt{\log T\,\log(e/\delta)} \ge \sqrt{\log2}$, we
  have $bR+b' = \mathcal{O}(\exp(\Lambda^{1/\beta})\sqrt{\log T\,\log(e/\delta)})$.
  Together with $\log(1/\delta)\le\Lambda$ and
  $\exp(2\Lambda^{1/\beta}) = E_\beta$, this bound gives
  \[
    \bF\mbar\log(1/\delta)
    \le \frac{2\exp(\Lambda^{1/\beta})\log(1/\delta)}{\Lambda}
    \cdot\frac{4t_0T(bR+b')K}{\mu_0}
    = \mathcal{O}\!\left(TE_\beta\sqrt{\log T\,\log(e/\delta)}\right).
  \]
  The other branch of $\alpha$ gives
  $16a_\beta K^2(T+t_0-1)\log(1/\delta)/\mu_0 \le 16a_\beta K^2t_0T\Lambda/\mu_0$,
  and $\Lambda \le (\beta/(2e))^\beta E_\beta$ because
  $u\exp(-2u^{1/\beta}) \le (\beta/(2e))^\beta$ for every $u>0$. Because
  $\log(e/\delta)\ge1$ and $\log T\ge\log2$, both branches are
  $\mathcal{O}(TE_\beta\sqrt{\log T\,\log(e/\delta)})$, and hence so is
  $2\alpha\log(1/\delta)$.

  For the third term, $2/\delta\le T/\delta$ gives
  $\exp(2\log^{1/\beta}(2/\delta))\le E_\beta$, so this term is
  $\mathcal{O}(TE_\beta)$.

  Dividing \eqref{eq:pl-preassembly} by $(T+t_0)(T+t_0-1)\ge T^2$ and bounding
  the first and third terms by the form of the second through $1/T^2\le1/T$,
  $E_\beta\ge1$, $\log(e/\delta)\ge1$ and $\log T\ge\log2$, we obtain that with
  probability at least $1-5\delta$,
  \[
    \cL_S(\bm{w}_{T+1}) - \cL_S(\bm{w}(S))
    = \mathcal{O}\!\left(\frac{E_\beta\sqrt{\log T\,\log(e/\delta)}}{T}\right).
  \]
  Replacing $\delta$ by $\delta/5$ changes $E_\beta$ into
  $\exp(2\log^{1/\beta}(5T/\delta))$ and $\log(e/\delta)$ into $\log(5e/\delta)$,
  both absorbed by the $\mathcal{O}(\cdot)$ by Lemma~\ref{lm:rescale} with $a=5$
  and $x=T/\delta\ge1$, which proves the theorem.
\end{proof}

\begin{proof}[Proof of Theorem~\ref{thm:pl-excess}]
  We retain the abbreviation $E_\beta = \exp(2\log^{1/\beta}(T/\delta))$ of
  \eqref{eq:etaabbr}. The population PL condition in Assumption~\ref{asm:pl}
  applied at $\bm{w}_{T+1}$,
  followed by $\|\bm{u}+\bm{v}\|^2 \le 2\|\bm{u}\|^2+2\|\bm{v}\|^2$, gives
  \begin{equation}
    \cL(\bm{w}_{T+1}) - \cL^\star
    \le \frac{1}{2\mu}\|\nabla \cL(\bm{w}_{T+1})\|^2
    \le \frac{1}{\mu}\left(
      \|\nabla \cL(\bm{w}_{T+1}) - \nabla \cL_S(\bm{w}_{T+1})\|^2
    + \|\nabla \cL_S(\bm{w}_{T+1})\|^2\right).
    \label{eq:pl-excess-split}
  \end{equation}
  Assumption~\ref{asm:smooth} makes $\cL_S$ $b$-smooth and $\bm{w}(S)$ minimizes
  it over $\cW$, so Lemma~\ref{lem:selfbound} and Theorem~\ref{thm:pl} bound the
  second term by, with probability $1-\delta$,
  \begin{equation}
    \|\nabla \cL_S(\bm{w}_{T+1})\|^2
    \le 2b\left(\cL_S(\bm{w}_{T+1}) - \cL_S(\bm{w}(S))\right)
    = \mathcal{O}\!\left(\frac{E_\beta\sqrt{\log T\,\log(e/\delta)}}{T}\right).
    \label{eq:pl-excess-opt}
  \end{equation}
  The first term is bounded by \eqref{eq:pl-gradgen}. Substituting it and
  \eqref{eq:pl-excess-opt} into \eqref{eq:pl-excess-split} and taking a union
  bound, we obtain with probability at least $1-2\delta$,
  \[
    \cL(\bm{w}_{T+1}) - \cL^\star
    = \mathcal{O}\!\left(E_\beta\left(
        \frac{\sqrt{\log T\,\log(e/\delta)}}{T}
    + \frac{\left(d+\log\tfrac1\delta\right)\log\tfrac e\delta\,\log T}{n}\right)\right).
  \]
  Let $c n\le T\le C n$ with constants $0<c\le C$ as in
  Section~\ref{sec:notation}, and set $a := \max\{C,1\}$. For the first term,
  $d\ge1$ and $\log(e/\delta)\ge1$ give
  $\sqrt{\log(e/\delta)} \le \log\tfrac e\delta
  \le \left(d+\log\tfrac1\delta\right)\log\tfrac e\delta$, $T\ge2$ gives
  $\sqrt{\log T}\le\log T/\sqrt{\log2}$, and $1/T\le1/(cn)$, so the first term
  is dominated by the second. For the second term, $T\le an$ and $n\ge2$ give
  $\log T \le \log a + \log n \le (1+\log a/\log2)\log n$, and
  Lemma~\ref{lm:rescale} with this $a$ and $x=n/\delta\ge1$ gives
  $E_\beta \le \exp(2\log^{1/\beta}a)\exp(2\log^{1/\beta}(n/\delta))$. The same
  lemma absorbs the replacement of $\delta$ by $\delta/2$.
\end{proof}

\section{Proof of the Bound for Clipped SGD}
\label{sec:proofs-clip}

Throughout this appendix $\tau$ denotes the clipping level
\eqref{eq:clip-level}, and we abbreviate
\[
  \bm{g}_t := \nabla \ell(\bm{w}_t; z_{j_t})
  = \nabla \cL_S(\bm{w}_t) + \epsilon_t, \qquad
  \bar{\bm{g}}_t := \EE[\tilde{\bm{g}}_t \mid \cF_{t-1}],
\]
with $\tilde{\bm{g}}_t = \bm{g}_t\min\{1,\tau/\|\bm{g}_t\|\}$ the clipped
gradient of Algorithm~\ref{alg:clip}. We write
$\Gamma_t := \min\{\tau\|\nabla \cL_S(\bm{w}_t)\|,\|\nabla \cL_S(\bm{w}_t)\|^2\}$
for the quantity summed in Theorem~\ref{thm:clip} and
$A_t := \{\|\nabla \cL_S(\bm{w}_t)\|\le\tau/2\}$,
and retain the abbreviations
$\Sigma_1,\Sigma_2$ of \eqref{eq:etaabbr}. The
iterate $\bm{w}_t$ is $\cF_{t-1}$-measurable, hence so are $\bar{\bm{g}}_t$,
$\Gamma_t$ and $\ind{A_t}$.

\begin{lemma}
  \label{lem:clip-tail}
  Let $s_\beta := g_\beta\left(\exp(2\log^{1/\beta}2)-1\right)$. Under
  Assumption~\ref{asm:gradnoise}, for every $t\in[T]$ and every $x\ge0$,
  \[
    \EE\left[\|\epsilon_t\|^2 \mid \cF_{t-1}\right] \le s_\beta K^2,
    \qquad
    \PP\!\left(\|\epsilon_t\|>x \mid \cF_{t-1}\right)
    \le 2\exp\!\left(-\log^\beta\!\left(\frac xK+1\right)\right),
  \]
  and if $\tau$ is given by \eqref{eq:clip-level}, then
  \begin{equation}
    \PP\!\left(\|\epsilon_t\|>\tau/8 \mid \cF_{t-1}\right) \le \delta/(4T) .
    \label{eq:clip-quantile}
  \end{equation}
\end{lemma}

\begin{proof}
  The pointwise inequality $1+u^2\le(1+u)^2$ for $u\ge0$ gives
  $\Psi_{\beta,2^{-\beta}}(u^2)\le\Psi_\beta(u)$. Taking conditional
  expectations and using Assumption~\ref{asm:gradnoise}, we obtain
  $\EE[\Psi_{\beta,2^{-\beta}}(\|\epsilon_t\|^2/K^2)\mid\cF_{t-1}]\le1$. Since
  $\Phi_\beta\le\Psi_{\beta,2^{-\beta}}$, the same bound holds with
  $\Phi_\beta$ in place of $\Psi_{\beta,2^{-\beta}}$, and $\Phi_\beta$ is convex
  and increasing, so Jensen's inequality gives
  $\Phi_\beta(\EE[\|\epsilon_t\|^2\mid\cF_{t-1}]/K^2)\le1$. Writing
  $r_\beta := \exp(2\log^{1/\beta}2)-1$ for the point at which
  $\Psi_{\beta,2^{-\beta}}(r_\beta)=1$, we have
  $\Phi_\beta(g_\beta r_\beta)\ge\Psi_{\beta,2^{-\beta}}(r_\beta)=1$, and the
  first bound follows because $s_\beta = g_\beta r_\beta$. For the second
  statement, conditional Markov's inequality and Assumption~\ref{asm:gradnoise} give
  \[
    \PP(\|\epsilon_t\|>x\mid\cF_{t-1})
    \le \frac{\EE[\Psi_\beta(\|\epsilon_t\|/K)+1\mid\cF_{t-1}]}
    {\Psi_\beta(x/K)+1}
    \le 2\exp\!\left(-\log^\beta(1+x/K)\right).
  \]
  By \eqref{eq:clip-level},
  $\log(1+\tau/(8K))\ge\log^{1/\beta}(8T/\delta)$, so the second bound with
  $x=\tau/8$ gives
  \[
    \PP(\|\epsilon_t\|>\tau/8\mid\cF_{t-1})
    \le 2\exp\!\left(-\log^\beta(1+\tau/(8K))\right)
    \le 2\exp(-\log(8T/\delta)) = \delta/(4T).
  \]
\end{proof}

\begin{proof}[Proof of Theorem~\ref{thm:clip}]
  We first relate the decrease in empirical risk to $\eta_t\Gamma_t$,
  keeping track of the bias and fluctuations of the clipped gradient.

  Since $\|\tilde{\bm{g}}_t\|\le\tau$, Lemma~\ref{thm:smooth} applied to the
  $b$-smooth function $\cL_S$ gives, for every $t\in[T]$,
  \begin{equation}
    \cL_S(\bm{w}_{t+1}) - \cL_S(\bm{w}_t)
    \le -\eta_t \inner{\tilde{\bm{g}}_t}{\nabla \cL_S(\bm{w}_t)} + \tfrac12 b \eta_t^2\tau^2 ,
    \label{eq:clip-descent}
  \end{equation}
  and we bound $-\eta_t\inner{\tilde{\bm{g}}_t}{\nabla \cL_S(\bm{w}_t)}$ separately on $A_t$
  and on its complement.

  On $A_t$, the empirical gradient norm is at most $\tau/2$, so clipping
  can occur only when the noise norm exceeds $\tau/2$. Clipping
  leaves $\tilde{\bm{g}}_t = \bm{g}_t$ when $\|\bm{g}_t\|\le\tau$ and gives
  $\|\tilde{\bm{g}}_t - \bm{g}_t\| = \|\bm{g}_t\|-\tau$ otherwise, so in both cases
  $\|\tilde{\bm{g}}_t - \bm{g}_t\| = \left(\|\bm{g}_t\|-\tau\right)_+$. Then we have
  \[
    \|\tilde{\bm{g}}_t - \bm{g}_t\|
    \le \left(\|\nabla \cL_S(\bm{w}_t)\|+\|\epsilon_t\|-\tau\right)_+
    \le \left(\|\epsilon_t\|-\tau/2\right)_+
    \le \|\epsilon_t\| \ind{\|\epsilon_t\|>\tau/2} .
  \]
  Taking conditional expectations and applying the Cauchy--Schwarz inequality,
  Lemma~\ref{lem:clip-tail} and $\tau/2\ge\tau/8$, on $A_t$ we obtain
  \begin{equation}
    \begin{aligned}
      \|\bar{\bm{g}}_t - \nabla \cL_S(\bm{w}_t)\|
      & = \left\|\EE\left[\tilde{\bm{g}}_t - \bm{g}_t \mid \cF_{t-1}\right]\right\| \\
      & \le \sqrt{\EE\left[\|\epsilon_t\|^2\mid\cF_{t-1}\right]
      \PP\!\left(\|\epsilon_t\|>\tau/2 \mid \cF_{t-1}\right)}
      \le \tfrac12\sqrt{s_\beta}\,K\sqrt{\frac{\delta}{T}} =: \bar B .
    \end{aligned}
    \label{eq:clip-bias}
  \end{equation}
  Decomposing $\inner{\tilde{\bm{g}}_t}{\nabla \cL_S(\bm{w}_t)} =
  \inner{\tilde{\bm{g}}_t-\bar{\bm{g}}_t}{\nabla \cL_S(\bm{w}_t)}
  + \inner{\bar{\bm{g}}_t-\nabla \cL_S(\bm{w}_t)}{\nabla \cL_S(\bm{w}_t)}
  + \|\nabla \cL_S(\bm{w}_t)\|^2$ and using \eqref{eq:clip-bias} and Young's inequality together with
  $\Gamma_t = \|\nabla \cL_S(\bm{w}_t)\|^2$ on $A_t$, we obtain on $A_t$
  \begin{equation}
    -\eta_t \inner{\tilde{\bm{g}}_t}{\nabla \cL_S(\bm{w}_t)}
    \le \xi_t - \tfrac12\eta_t\Gamma_t + \tfrac12\eta_t \bar B^2,
    \qquad
    \xi_t := -\eta_t\inner{\tilde{\bm{g}}_t-\bar{\bm{g}}_t}{\nabla \cL_S(\bm{w}_t)}\ind{A_t}.
    \label{eq:clip-caseA}
  \end{equation}
  The centering by $\bar{\bm{g}}_t$ makes $\xi_t$ a martingale difference:
  the factor $\eta_t\nabla\cL_S(\bm{w}_t)\ind{A_t}$ is
  $\cF_{t-1}$-measurable. Thus \eqref{eq:clip-caseA} separates the descent
  term from a centered fluctuation and a deterministic bias bound.

  Outside $A_t$ we work on the event
  $\mathcal{E} := \bigcap_{t\in[T]}\{\|\epsilon_t\|\le\tau/8\}$, which has
  $\PP(\mathcal{E}^c)\le\delta/4$ by \eqref{eq:clip-quantile} and a union
  bound. If $\|\nabla \cL_S(\bm{w}_t)\|>\tau/2$, then
  $\|\epsilon_t\|\le\tau/8\le\|\nabla \cL_S(\bm{w}_t)\|/4$ on $\mathcal{E}$, so
  $\inner{\bm{g}_t}{\nabla \cL_S(\bm{w}_t)} \ge \tfrac34\|\nabla \cL_S(\bm{w}_t)\|^2$
  and $\tfrac34\|\nabla \cL_S(\bm{w}_t)\| \le \|\bm{g}_t\|
  \le \tfrac54\|\nabla \cL_S(\bm{w}_t)\|$, whence
  \begin{align*}
    \inner{\tilde{\bm{g}}_t}{\nabla \cL_S(\bm{w}_t)}
    & = \frac{\min\{\tau,\|\bm{g}_t\|\}}{\|\bm{g}_t\|}
    \inner{\bm{g}_t}{\nabla \cL_S(\bm{w}_t)}                                       \\
    & \ge \frac{\min\{\tau,\tfrac34\|\nabla \cL_S(\bm{w}_t)\|\}}
    {\tfrac54\|\nabla \cL_S(\bm{w}_t)\|}\cdot \tfrac34\|\nabla \cL_S(\bm{w}_t)\|^2 \\
    & = \tfrac35\min\left\{\tau\|\nabla \cL_S(\bm{w}_t)\|,\
    \tfrac34\|\nabla \cL_S(\bm{w}_t)\|^2\right\}
    \ge \tfrac25\Gamma_t .
  \end{align*}
  Since $\xi_t=0$ outside $A_t$ and $\tfrac25\le\tfrac12$, this bound and
  \eqref{eq:clip-caseA} together give, on $\mathcal{E}$ and for every
  $t\in[T]$,
  \begin{equation}
    -\eta_t \inner{\tilde{\bm{g}}_t}{\nabla \cL_S(\bm{w}_t)}
    \le \xi_t - \tfrac25\eta_t\Gamma_t + \tfrac12\eta_t\bar B^2 .
    \label{eq:clip-step}
  \end{equation}
  Substituting \eqref{eq:clip-step} into \eqref{eq:clip-descent}, summing over
  $t\in[T]$ and using $\cL_S(\bm{w}_{T+1}) \ge \cL_S(\bm{w}(S))$, we obtain on
  $\mathcal{E}$
  \begin{equation}
    \tfrac25\sum_{t\in[T]}\eta_t\Gamma_t
    \le \cL_S(\bm{w}_1) - \cL_S(\bm{w}(S)) + \sum_{t\in[T]}\xi_t
    + \tfrac12\bar B^2\Sigma_1 + \tfrac12 b\tau^2\Sigma_2 .
    \label{eq:clip-master}
  \end{equation}

  We bound $\sum_{t\in[T]}\xi_t$ by Proposition~\ref{thm:freedman} with confidence parameter $\delta/4$ and
  $K_{t-1} := c_\beta\eta_t\|\nabla \cL_S(\bm{w}_t)\|K\ind{A_t}$, where
  $c_\beta := 2+\tfrac14\sqrt{s_\beta}$. Condition (A2) holds because
  $\eta_t,\nabla \cL_S(\bm{w}_t), \ind{A_t}$ are $\cF_{t-1}$-measurable. For (A3),
  the display preceding \eqref{eq:clip-bias} gives
  $\|\tilde{\bm{g}}_t-\nabla \cL_S(\bm{w}_t)\| \le \|\tilde{\bm{g}}_t-\bm{g}_t\| +
  \|\epsilon_t\| \le 2\|\epsilon_t\|$ on $A_t$, hence
  $\|\tilde{\bm{g}}_t-\bar{\bm{g}}_t\| \le 2\|\epsilon_t\| + \bar B$ and
  $|\xi_t| \le \eta_t\|\nabla \cL_S(\bm{w}_t)\|(2\|\epsilon_t\|+\bar B)\ind{A_t}$ by the
  Cauchy--Schwarz inequality. Since $T\ge4$ and $\delta<1$ give
  $\sqrt{\delta/T}\le\tfrac12$, the bias satisfies
  $\bar B \le \tfrac14\sqrt{s_\beta}K$ by \eqref{eq:clip-bias}, and the
  following direct estimate verifies (A3). Since $c_\beta=2+\tfrac14\sqrt{s_\beta}$,
  on the $\cF_{t-1}$-measurable
  event $\{K_{t-1}>0\}$,
  \[
    \frac{|\xi_t|}{K_{t-1}}
    \le \frac{2}{c_\beta}\frac{\|\epsilon_t\|}{K}
    +\frac{c_\beta-2}{c_\beta}.
  \]
  Since $\Psi_\beta$ is increasing and convex, and
  $\Psi_\beta(1)=\exp((\log 2)^\beta)-1\le1$ for $\beta>1$, we have on this event
  \[
    \EE\!\left[\Psi_\beta\!\left(\frac{|\xi_t|}{K_{t-1}}\right)
    \;\middle|\;\cF_{t-1}\right]
    \le \frac{2}{c_\beta}
    \EE\!\left[\Psi_\beta\!\left(\frac{\|\epsilon_t\|}{K}\right)
    \;\middle|\;\cF_{t-1}\right]
    +\frac{c_\beta-2}{c_\beta}\Psi_\beta(1)
    \le 1
  \]
  by Assumption~\ref{asm:gradnoise}. The degenerate case $K_{t-1}=0$ forces
  $\xi_t=0$ almost surely, as in the proof of Theorem~\ref{thm:main1}. For
  (A1), the definition of $A_t$ and $\eta_t\le\eta_1$ give the deterministic
  bound $0 \le K_{t-1} \le \tfrac12 c_\beta\eta_1\tau K =: \mbar$, which is where
  Assumption~\ref{asm:cap} entered the proof of Theorem~\ref{thm:main1}.

  To absorb the fluctuation bound into the left-hand side of
  \eqref{eq:clip-master}, we now bound its variance proxy $V_T$ by the
  same sum $\sum_{t\in[T]}\eta_t\Gamma_t$. Since
  $\Gamma_t = \|\nabla \cL_S(\bm{w}_t)\|^2 $ on $A_t$, we have
  \begin{equation}
    V_T = \sum_{t\in[T]} a_\beta K_{t-1}^2
    = c_\beta^2 a_\beta K^2 \sum_{t\in[T]}\eta_t^2\|\nabla \cL_S(\bm{w}_t)\|^2\ind{A_t}
    \le c_\beta^2 a_\beta K^2 \eta_1 \sum_{t\in[T]}\eta_t\Gamma_t .
    \label{eq:clip-V}
  \end{equation}
  Setting $\alpha := \max\{\bF\mbar,\, 10c_\beta^2 a_\beta K^2\eta_1\}$ and applying
  Proposition~\ref{thm:freedman}(i) with $\gamma=0$, $\lambda=1/(2\alpha)$ and
  $x = 2\alpha\log(4/\delta)$ bounds the probability of
  $\bigcup_{k\in[T]}\{S_k\ge x \text{ and } V_k \le \alpha S_k\}$ by
  $3\delta/4$. On the complement every $k\in[T]$
  satisfies $S_k \le \max(x, V_k/\alpha) \le x + V_k/\alpha$, so
  \eqref{eq:clip-V} and $\alpha\ge10c_\beta^2 a_\beta K^2\eta_1$ give
  $\sum_{t\in[T]}\xi_t \le 2\alpha\log(4/\delta) +
  \tfrac1{10}\sum_{t\in[T]}\eta_t\Gamma_t$ with probability at least
  $1-3\delta/4$. Substituting this into \eqref{eq:clip-master} and taking a
  union bound with $\mathcal{E}$, with probability at least $1-\delta$,
  \begin{equation}
    \tfrac{3}{10}\sum_{t\in[T]}\eta_t\Gamma_t
    \le \cL_S(\bm{w}_1) - \cL_S(\bm{w}(S)) + 2\alpha\log(4/\delta)
    + \tfrac12\bar B^2\Sigma_1 + \tfrac12 b\tau^2\Sigma_2 .
    \label{eq:clip-assembled}
  \end{equation}

  It remains to bound the terms in \eqref{eq:clip-assembled}.
  By the initial-gap assumption in Section~\ref{sec:main-results} and
  $\bm{w}_1=\bm{0}$, the initial gap is at most
  the deterministic constant $\Delta_0$ almost surely.
  Since $\Sigma_1\le\eta_1 T$, the bias contribution
  $\tfrac12\bar B^2\Sigma_1$ is at most $s_\beta K^2\eta_1\delta/8$,
  while the smoothness remainder equals
  $32bK^2\exp(2\log^{1/\beta}(8T/\delta))\Sigma_2$ by \eqref{eq:clip-level}. For
  the concentration term $2\alpha\log(4/\delta)$,
  \eqref{eq:freedef} with $N=T$ and confidence parameter $\delta/4$
  has $\Lambda = \log(4T/\delta)$ and
  \[
    \bF\mbar
    = \frac{2\left(\exp(\Lambda^{1/\beta})-1\right)}{\Lambda}\cdot
    \tfrac12 c_\beta\eta_1\tau K
    \le \frac{8c_\beta\eta_1K^2\exp(\Lambda^{1/\beta})
    \exp\!\left(\log^{1/\beta}(8T/\delta)\right)}{\Lambda},
  \]
  so $\log(4/\delta)\le\Lambda$ for the first branch of $\alpha$,
  $\log(4/\delta)\le(1+\log4)\log(e/\delta)$ for the second, and
  Lemma~\ref{lm:rescale} with $a=4$ and with $a=8$ give
  $2\alpha\log(4/\delta) = \mathcal{O}(\exp(2\log^{1/\beta}(T/\delta))
  + a_\beta K^2\eta_1\log(e/\delta))$. Every term of
  \eqref{eq:clip-assembled} is therefore
  $\mathcal{O}(\exp(2\log^{1/\beta}(T/\delta))(\log(e/\delta)+\Sigma_2))$, the
  additive constants, including $\Delta_0$, being absorbed because
  $\log(e/\delta)\ge1$.
\end{proof}

Implementing the clipping level in \eqref{eq:clip-level} requires the noise
parameters $K$ and $\beta$. Once these parameters, $T$, $\delta$, and the
step-size schedule are fixed, clipping provides a deterministic bound on
the trajectory without Assumption~\ref{asm:cap}. Indeed,
$\|\tilde{\bm{g}}_t\|\le\tau$ and $\bm{w}_1=\bm{0}$ imply, for every
realization and every $t\in[T+1]$,
\[
  \|\bm{w}_t\|
  \le \sum_{s=1}^{t-1}\eta_s\|\tilde{\bm{g}}_s\|
  \le \tau\sum_{s=1}^{t-1}\eta_s
  \le \tau\sum_{s\in[T]}\eta_s.
\]
Thus all clipped iterates lie in a ball centered at the initialization with
radius $R=\tau\sum_{s\in[T]}\eta_s$ known before training. This radius can be
used when applying the uniform gradient bound of Lemma~\ref{lm:gradgen} to
the clipped trajectory, under that lemma's hypotheses. It plays the role
of the radius that, for SGD, is only bounded probabilistically
in the proof of Theorem~\ref{thm:main2}.
\section{Bounded Noise and Dependence on the Tail Parameter}
\label{sec:bounded-noise-scales}

Even for the same bounded noise, the choice of $\beta$ affects the
constants in our convergence guarantee. The following calculation
parallels the discussion immediately after Remark~7 in
\citet{madden2024high}.

Suppose that $\|\epsilon_t\|\le\rho$ almost surely for every $t$, with a
deterministic $\rho>0$ uniform over datasets and iterates. This is stronger
than finite support at each iterate. Fixing $\rho$, for every $\beta>1$
a sufficient Orlicz scale is
\begin{equation}
  K(\beta,\rho)=\frac{\rho}{\exp((\log2)^{1/\beta})-1}.
  \label{eq:bounded-beta-scale}
\end{equation}
Indeed, $\Psi_\beta(\rho/K(\beta,\rho))=1$, so boundedness implies the conditional
Orlicz bound in Assumption~\ref{asm:gradnoise}. As $\beta$ increases,
$K(\beta,\rho)$ decreases from $\rho$ to the positive limit $\rho/(e-1)$.

However, a smaller scale does not by itself give a sharper convergence
bound. Under the remaining hypotheses of Theorem~\ref{thm:main1}, the
choice of $\alpha$ in \eqref{eq:bound1} also involves the coefficient
\begin{equation}
  a_\beta K(\beta,\rho)^2
  =\frac{8\rho^2\exp(\beta^\beta/2)J_\beta}
  {[\exp((\log2)^{1/\beta})-1]^2},
  \label{eq:bounded-beta-coefficient}
\end{equation}
where $J_\beta$ is defined in \eqref{eq:freedef}. Since
$J_\beta\ge\int_0^1 e^{2s-1/2}\,ds>0$, this coefficient diverges as
$\beta\to\infty$, despite the decrease in $K(\beta,\rho)$.

Thus, choosing a larger $\beta$ solely to reduce the displayed tail
factors can substantially loosen the guarantee once its constants are
included. This concerns the bound supplied by our proof; the noise law
and the SGD algorithm remain unchanged. Boundedness therefore does not
make the choice of $\beta$ immaterial: comparing guarantees requires
accounting for both the Orlicz scale and the convergence constants.
\end{document}